\documentclass[smallcondensed]{svjour3} 

\usepackage{graphicx}
\usepackage{natbib}

\usepackage[utf8]{inputenc}
\usepackage[ruled,linesnumbered]{algorithm2e}

\usepackage[breaklinks]{hyperref}
\usepackage{breakcites}
\usepackage[T1]{fontenc}
\usepackage[english]{babel}
\usepackage{xspace}
\usepackage{amsmath}
\usepackage{amsfonts}
\usepackage{amssymb}
\usepackage{stmaryrd}
\usepackage{lipsum}
\usepackage{nicefrac}

\usepackage{pgfplots}
\usepackage{booktabs}

\usepackage{tikz}
\usetikzlibrary{automata, positioning, calc, shapes, arrows, shadows, patterns, plotmarks}
\usepgflibrary{shapes}

\usepackage{tgpagella}   

\newcommand{\revisedC}[1]{{\color{black} #1}}

\newcommand\held[1]{{\fcolorbox{black}{white}{$#1$}}}
\newcommand\heldb[1]{{\fcolorbox{white}{lightgray}{$#1$}}}
\newcommand\heldc[1]{{\underline{#1}}}

\newcommand\hidden[1]{{}}

\newcommand{\agentSet}{\mathcal{N}}
\newcommand{\resourceSet}{\mathcal{R}}
\newcommand{\proc}{M}
\newcommand{\swapfamily}{\mathcal{M}_2}
\newcommand{\tuple}[1]{\langle #1 \rangle}

\newtheorem{observation}{Observation}

\hypersetup{%
	colorlinks=true,
	citecolor=red,
	linkcolor=red,
	anchorcolor=red,
	urlcolor=red,
	pdfstartview=FitH,
	pdfdisplaydoctitle=true,
	pdftoolbar=true,
	pdfmenubar=true
}

\begin{document}
	
\title{Swap Dynamics in Single-Peaked Housing Markets}

\author{Aur{\'e}lie Beynier \and Nicolas Maudet \and \\ Simon Rey \and Parham Shams}
\authorrunning{  }

\institute{ A. Beynier, N. Maudet and P. Shams \at Sorbonne Universit{\'e}, CNRS,\\
	Laboratoire d'Informatique de Paris 6, LIP6,\\
	F-75005 Paris, France\\
	\email{\{aurelie.beynier, nicolas.maudet, parham.shams\}@lip6.fr}\\
	\and S. Rey 
	\at Institute for Logic, Languages and Computation (ILLC)\\
	Amsterdam, the Netherlands\\
	\email{s.j.rey@uva.nl}}

\date{Received: date / Accepted: date}

%
\maketitle              

\begin{abstract}
	This paper focuses on the problem of fairly and efficiently allocating resources to agents. 
	We consider a specific setting, usually referred to as a \emph{housing market}, where each agent must receive exactly one resource (and initially owns one). 
	 In this framework, in the domain of linear preferences, the Top Trading Cycle (TTC) algorithm  is the only procedure satisfying Pareto-optimality, individual rationality and strategy-proofness. Under the restriction of single-peaked preferences,  Crawler enjoys the same properties. These two centralized procedures might however involve long trading cycles. In this paper we focus instead on procedures involving the shortest cycles: bilateral swap-deals. In such swap dynamics, the agents perform pairwise mutually improving deals until reaching a swap-stable allocation (no improving swap-deal is possible). We prove that in the single-peaked domain every swap-stable allocation is Pareto-optimal, showing the efficiency of the swap dynamics. In fact, this domain turns out to be maximal when it comes to guaranteeing this property. Besides, both the outcome of TTC and  Crawler can always be reached by sequences of swaps. However, some Pareto-optimal allocations are not reachable through improving swap-deals. 
	We further analyze the outcome of swap dynamics through social welfare notions, in our context the average or minimum rank of the resources obtained by agents in the final allocation. We start by providing a worst-case analysis of these procedures. Finally, we present an extensive experimental study in which different versions of swap dynamics are compared to other existing allocation procedures. We show that they exhibit good results on average in this domain, under different cultures for generating synthetic data.  
\end{abstract}
	

\section{Introduction}

\label{sec:intro}

This paper studies the problem of allocating fairly and efficiently a set of indivisible resources to a set of agents.  
	We consider a specific setting, usually referred to as a \emph{housing market} \citep{shapley1974cores}, where each agent must receive exactly one resource (and initially owns one). 
	We investigate housing markets by moreover assuming that the preferences are \emph{single-peaked} \citep{arrow1951social,black1948rationale}.
	Under this assumption, there exists a common ordering of the resources such that 
	the further you go from your most preferred resource---either going left or right according to the ordering---the less preferred the resources are. 
	This domain restriction would for instance correspond to real-world situations where the preferences of agents are related to a one-dimensional space, such as the temperature in a room, the distance of a house to a facility (metro, bike rental station, supermarket), the price of a new product, the size of a cloth  \citep{SPcircle, Puppe}.  
	In voting, this restriction domain may be appropriate in political elections where  preferences usually decrease when considering candidates further away on the political spectrum from the most preferred candidate.

The design of resource allocation procedures is primarily guided by the  
the properties we want the final allocation, or the procedure itself, to satisfy.  
In single-peaked housing markets, the two prominent allocation procedures are \emph{Top Trading Cycle} (TTC) \citep{shapley1974cores} and \emph{Crawler} \citep{bade2019matching}. They both satisfy a set of key desirable properties: Pareto-optimality, strategy-proofness and individual rationality \citep{roth1982incentive,shapley1974cores,ma1994strategy,bade2019matching}. These three properties thus provide strong guarantees in terms of efficiency (Pareto-optimality) and incentive compatibility (individual rationality and strategy-proofness). In fact, Crawler has an additional advantage over TTC in terms of strategy-proofness, since it can implemented in obviously dominant strategies \citep{bade2019matching}. 

	However, both TTC and  Crawler 
	require a significant amount of global coordination. Indeed, they are both based on \emph{trading cycles}---cyclic exchanges of resources between several agents---which can potentially involve many agents. Such long cycles may not be acceptable or even feasible in practice, for instance because of the risk of failure they induce. 
	As an example, kidney exchange programs usually restrict the size of the cycles to two or three \citep{roth2005pairwise} because of time constraints.

Limiting the length of these cycles is thus a relevant agenda to reduce the coordination complexity and make procedures more robust.  
It is in particular a prerequisite to the development of more \emph{decentralized} approaches that rely on the agents autonomously performing simple deals, without being hampered by a prohibitive coordination and communication cost.

Indeed, the use of a central coordinator induces a weak point in the system: the coordinator causes a bottleneck whose default often leads to the failure of the whole allocation process. 
The system may also be inherently distributed and the use of a central coordinator may not be possible because of limitations in the communication infrastructure or because of privacy requirements. In addition, centralized procedures may be perceived as less fair by the agents who are not allowed to actively take part in the allocation process \citep{van1997judge,leventhal1980should,thibaut1975procedural}. 
While TTC and  Crawler are both presented as centralized procedures (\emph{i.e.} a central authority computes each trading cycle that should be implemented, and dictates it to the agents\footnote{In principle, it is possible to distribute the execution of central procedures, by letting all agents broadcast, compute locally their own version of the central algorithm, and coordinate. This might however induce a significant cost.}),  distributed procedures \citep{sandholm1998contract,chevaleyre2017distributed} take a different perspective. Agents autonomously negotiate over the
resources and locally agree on deals, and the outcome of the resource allocation problem results 
from the sequence of such local deals. 
Such dynamics may achieve interesting efficiency and fairness results \citep{chevaleyre2007reaching,endriss2006negotiating}. However, long cycles still pose real challenges for these procedures \citep{rosenschein1994rules} as they involve distributed coordination among numerous agents. 

In this paper, we pursue this line of research and focus on the simplest possible local deals in the context of housing markets:  \emph{swap-deals} \citep{damamme2015power}. Under a \emph{swap dynamics}, agents meet each other, in a pairwise fashion, and exchange their resources if they both benefit from it. The process iterates until a stable state, an equilibrium, is reached. 
Under this approach, once a trading cycle is selected, very little coordination is required. 
Nevertheless, selecting \emph{which} swap-deal (with which agent) to perform may still require significant prior coordination, depending on the heuristic used. Hence this approach is best described as a family of dynamics, with different degrees of decentralization depending on the heuristics used to select the improving deals to be implemented. We shall study several of them in this paper.

\paragraph{Contributions.} 

	We first establish that a large class of swap dynamics is vulnerable to manipulation, unveiling a tension between efficiency and strategy-proofness for such decentralized procedures. 
	On the positive side, we demonstrate that in housing markets with single-peaked preferences, every allocation that is stable with respect to swap-deals is Pareto-optimal. 
We also prove that the single-peaked domain is maximal in that respect: in other words, any larger domain would fail to offer such a guarantee. 
Moreover, even though some Pareto-optimal allocations may not be reached by a sequence of improving swap-deals, we show that the allocations returned by TTC and  Crawler are both reachable: there exist sequences of swap-deals simulating these procedures, such that (potentially) long trading cycles are not necessary any longer.  
We further investigate how swap dynamics behave with respect to social welfare, taken as the average or minimum rank of the resource obtained by the agents. 
After a worst-case analysis, we explore experimentally the influence of different heuristics used to select deals. These experiments highlight that swap dynamics perform particularly well with respect to these objectives.

\paragraph{Outline of the paper.} 

We first present some related works (Section \ref{sec:related}) and introduce our model and some additional definitions in Section \ref{sec:preliminaries}. Section \ref{sec:centralProcedures} describes the centralized allocation procedures discussed in this paper, while swap dynamics are introduced in Section \ref{sec:decentralProcedures}. A formal analysis of these procedures  is presented in Section \ref{sec:swapEfficiency}. Section \ref{sec:pricesOf} offers a comparison between swap dynamics outcomes and social welfare optimization, in a ``price of anarchy'' perspective. The experimental analysis is presented in Section \ref{sec:expe}. Section \ref{sec:conclu} concludes. 

\section{Related works} 
\label{sec:related}

The problem of fair and efficient allocation of indivisible goods is well studied in economics, with original developments impulsed by the computer science perspective in recent years \citep{bouveret2016fair,Aziz20,Walsh20,thomson2016introduction,moulin2018fair}.  
In the present paper we focus on the model defined by \cite{shapley1974cores}, called \emph{housing market} or \emph{assignment problem}, in which each agent must receive exactly one indivisible resource. \cite{shapley1974cores} defined the \emph{Top-Trading Cycle} (TTC) algorithm which has been  extensively studied \citep{roth1982incentive} and shown to be the only mechanism satisfying Pareto-optimality when preferences are strict \citep{ma1994strategy}. \cite{bade2019matching} explored another direction by restricting preferences to single-peaked domains. In this case she presents the \emph{Crawler} that satisfies the same properties as TTC. This work has been further explored by \citet{liu2018large}.

Following this line of work, we assume single-peaked preferences in this paper. This domain of preferences has been introduced by \cite{black1948rationale} and \cite{arrow1951social}. Initially motivated in voting contexts, it is now a well studied domain of preferences \citep{moulin1991axioms,elkind2017structured}. Numerous works have explored single-peaked preferences in the context of fair division. \cite{sprumont1991division} for instance characterized the \emph{uniform allocation rule}, the unique strategy-proof, efficient and anonymous allocation procedure with single-peaked preferences and divisible resources. \cite{kasajima2013probabilistic} investigated probabilistic allocation of indivisible resources with single-peaked preferences. More recently, \cite{hougaard2014assigning} extended this research area to indivisible resources and considered the problem of assigning agents to a line under single-peaked preferences.
	On the empirical side, it is not clear wether single-peaked preferences can be observed in real-life scenarios, as mentioned in \cite{Puppe}. On the one hand, \cite{spector2000rational}, \cite{demarzo2003persuasion} and \cite{list2013deliberation} argue that interactions between the agents lead towards single-peaked preferences. On the other hand, \cite{egan2014something} shows that political preferences might be double-peaked (and not single-peaked) when it comes to some polarizing topics.

Following the development of multi-agent systems, decentralized procedures have been defined through the idea of local exchanges between agents. 
\cite{sandholm1998contract} considered the problem of reallocating tasks among individually rational agents. \cite{aziz2016optimal} investigated different complexity problems in this setting. \cite{endriss2006negotiating} and \cite{chevaleyre2010simple} respectively characterized the class of deals and the class of preferences required to reach socially optimal allocations, while communication complexity issues were explored in \citep{endriss2005communication,dunne}. 
\cite{chevaleyre2007reaching} 
focused on reaching efficient and envy-free allocations. Similar procedures were also introduced in the area of two-sided matching \citep{roth1990random,ackermann2011uncoordinated}. 
The idea of using swap-deals was explored for instance by \cite{abbassi2013fair}, who studied barter exchange networks. \cite{gourves2017object}, \cite{saffidine2018constrained} and \cite{huang2019object} studied dynamics of swap-deals by considering an underlying social network constraining the possible interactions of the agents, albeit in a centralized perspective, and focusing on complexity issues. These results were recently extended by \cite{BentertEtAl2019}. 
\cite{brandt2019convergence} studied the convergence of swap-deals to Pareto-optimal allocations in housing, marriage and roommate markets.

\section{Preliminaries}
\label{sec:preliminaries}

We start by presenting the basic components of our model and the different fairness and efficiency concepts discussed in this paper.

\subsection{The model}

Let us consider a set $\agentSet = \{a_1, \ldots, a_n\}$ of $n$ agents and a set $\resourceSet = \{r_1, \ldots, r_n\}$ of $n$ resources. An allocation $\pi = \tuple{\pi_{a_1}, \ldots, \pi_{a_n}}$ is a vector of $\resourceSet^n$ such that $\bigcup_{a_i \in \agentSet} \pi_{a_i} = \resourceSet$. A element $\pi_{a_i} \in \resourceSet$ of $\pi$ represents the unique resource allocated to $a_i$ (recall that we consider housing markets where exactly one resource is assigned to each agent). The set of all allocations is denoted $\Pi$. 

Following previous work \citep{shapley1974cores,brams2003fair,bouveret2010fair,aziz2015fair}, agents are assumed to express their preferences over the resources through complete linear orders. We denote by $\succ_{i}$ agent $a_i$'s preferences, where $r_1 \succ_{i} r_2$ means that $a_i$ prefers $r_1$ over $r_2$. 

For a given linear order $\succ$, we use $top(\succ)$ to denote its top-ranked resource: $\forall r \in \resourceSet \backslash\{top(\succ)\}, top(\succ) \succ r$. Similarly, $snd(\succ)$ refers to the second most preferred resource in $\succ$. With a slight abuse of notation we will write $top(a_i)$ and $snd(a_i)$ to refer to $top(\succ_{i})$ and $snd(\succ_{i})$. When it is not clear from the context we will subscript these notations to specify the resource set. For instance $top_R(a_i)$ is the most preferred resource for agent $a_i$ among resources in $R \subseteq \resourceSet$.

Given a resource $r \in \resourceSet$ and an agent $a_i \in \agentSet$, we use $rank_{a_i}(r)$ to refer to the rank of $r$ in $\succ_{i}$. We have then $rank_{a_i}(top(a_i)) = n$, $rank_{a_i}(snd(a_i)) = n - 1$, etc...

For a given agent $a_i \in \agentSet$, $a_i$ strictly prefers $\pi$ to $\pi'$ (denoted as $\pi \succ_i \pi'$) iff $\pi_{a_i} \succ_i \pi'_{a_i}$. 
We also introduce weak preferences over an allocation. For a given agent  $a_i \in \agentSet$, $a_i$ weakly prefers $\pi$ to $\pi'$  (written $\pi \succeq_i \pi'$) iff $\pi_{a_i} \succ_i \pi'_{a_i}$ or $\pi_{a_i} = \pi'_{a_i}$  

Finally, a preference profile $L = \tuple{\succ_{i} \mid a_i \in \agentSet}$ is a vector of $n$ linear orders, one for each agent. An instance of a resource allocation problem is then a vector $I = \tuple{\agentSet, \resourceSet, L, \pi^0}$ composed of a set of agents $\agentSet$, a set of resources $\resourceSet$, a preference profile $L$ and an initial allocation $ \pi^0$. 

\medskip

In some settings, natural properties of  the agents' preferences can be identified, thus restricting the set of possible preference orderings. The notion of preference domain formalizes these restrictions. For a set of resources $\resourceSet$, we denote by $\mathcal{L}_\resourceSet$ the set of all linear orders over $\resourceSet$. Any subset $D \subseteq \mathcal{L}_\resourceSet$ is then called a preference domain.

We say that an instance $I = \tuple{\agentSet, \resourceSet, L, \pi^0}$ is defined over a preference domain $D$ if the preferences of the agents are selected inside $D$, that is, for every $\succ_i \in L$, we have $\succ_i \in D$. The set of all instances defined over $D$ is denoted $\mathcal{I}_D$. 

\subsection{Single-peaked preferences}

Under single-peaked preferences \citep{black1948rationale,arrow1951social}, the agents are assumed to share a common axis $\lhd$ over the resources and  individual rankings are defined with respect to this axis.

\begin{definition}
	Let $\resourceSet$ be a set of resources and $\lhd$ an axis  (i.e. a linear order) over $\resourceSet$. We say that a linear order $\succ$ is single-peaked with respect to $\lhd$ if we have:
	$$\forall (r_1, r_2) \in \resourceSet^2 \text{ s.t. } \enspace \left.\begin{array}{r}
	r_2 \lhd r_1 \lhd top(\succ), \\ 
	\text{or, } top(\succ) \lhd r_1 \lhd r_2
	\end{array}\right\} \Rightarrow r_1 \succ r_2.$$
\end{definition}

In words, $\succ$ is single-peaked with respect to $\lhd$ if $\succ$ is decreasing on both the left and the right sides of $top(\succ)$, where left and right are defined by $\lhd$.

For a given axis $\lhd$, we call $\mathcal{SP}_\lhd$ the set of all the linear orders single-peaked with respect to $\lhd$:
$$\mathcal{SP}_\lhd = \{\succ{} \in \mathcal{L}_\resourceSet \mid{} \succ \text{ is single-peaked w.r.t. } \lhd\}.$$

A preference domain $D$ is called single-peaked if and only if there exists an axis $\lhd$ such that $D \subseteq \mathcal{SP}_\lhd$. An instance $I$ is said to be single-peaked if it is defined over a single-peaked preference domain. 

\medskip

Based on these definitions, it is straightforward that restricting preference domains preserves single-peakedness.

\begin{observation}
	\label{obs:restricted-sp}
	Let $D$ be a preference domain single-peaked over $\lhd$. For a subset of resources $\resourceSet' \subseteq \resourceSet$, the domain $D'$ defined as the restriction of $D$ to $\resourceSet'$ is a single-peaked domain over $\lhd'$, the restriction of $\lhd$ to $\resourceSet'$.
\end{observation}

\cite{ballester2011characterization} provided a characterization of single-peaked domains. In particular, they gave a necessary condition for a domain to be single-peaked: it should be worst-restricted \citep{sen1966possibility}.

\begin{definition}
	\label{def:worst-restricted}
	An instance $I = \tuple{\agentSet, \resourceSet, L, \pi^0}$ is worst-restricted if for any triplet of resources $(r_x, r_y, r_z) \in \resourceSet^3$, one of them is never ranked last in the restriction of $L$ to these three resources.
\end{definition}

\begin{proposition}[\cite{ballester2011characterization}]
	\label{prop:SP=>WR}
	If an instance is single-peaked then it is worst-restricted.
\end{proposition}

Let us illustrate the single-peaked domain with a simple example.

\begin{example}
	Consider the following four linear orders defined over 3 resources.
	\begin{align*}
	\succ_1: \enspace & r_1 \succ_{1} r_2 \succ_{1} r_3 \\
	\succ_2: \enspace & r_3 \succ_{2} r_2 \succ_{2} r_1 \\
	\succ_3: \enspace & r_2 \succ_{3} r_1 \succ_{3} r_3 \\
	\succ_4: \enspace & r_2 \succ_{4} r_3 \succ_{4} r_1
	\end{align*}
	One can check that these linear orders represent a single-peaked preference profile with respect to $\lhd$ defined as: $r_1 \lhd r_2 \lhd r_3$. In fact these orders exactly correspond to $\mathcal{SP}_\lhd$. However, the following preferences are not single-peaked:
	\begin{align*}
	\succ_1: \enspace & r_1 \succ_{1} r_2 \succ_{1} r_3 \\
	\succ_2: \enspace & r_3 \succ_{2} r_1 \succ_{2} r_2 \\
	\succ_3: \enspace & r_2 \succ_{3} r_3 \succ_{3} r_1
	\end{align*}
	It can be checked that there is no linear order $\lhd$ over which these preferences are single-peaked. Indeed, they are not worst-restricted: every resource of the triplet $(r_1, r_2, r_3)$ is ranked last at least once, which violates Proposition \ref{prop:SP=>WR}.
\end{example}

\subsection{Efficiency and fairness criteria of an allocation}

There exists an extensive literature investigating how to define efficiency requirements for an allocation (see \citep{chevaleyre2006issues} and \citep{thomson2016introduction} for some surveys). 
Pareto-optimality is well-known basic efficiency requirement. 

\begin{definition}[Pareto-optimality]
	Let $\pi$ and $\pi'$ be two allocations. We say that $\pi'$ Pareto-dominates $\pi$ if and only if:
	\begin{equation*}
	\forall a_i \in \agentSet, \enspace  \pi' \succeq_i \pi, \textrm{~and ~}
	\exists a_j \in \agentSet, \enspace  \pi' \succ_j \pi.
	\end{equation*}
	
	An allocation $\pi$ is then said to be Pareto-optimal if there is no other allocation $\pi'$ that Pareto-dominates $\pi$.
\end{definition}

Efficiency can also be evaluated via social welfare measures.
One such measure is the \emph{average rank} ($ark$) of the resources held by the agents, defined as:
$$ark(\pi) = \frac{1}{n} \cdot \sum_{a_i \in \agentSet} rank_{a_i}(\pi_{a_i}).$$
Maximizing the average rank is of course equivalent to maximizing the sum of ranks, and can also be interpreted as the utilitarian social welfare, under the assumption that agents have Borda  utilities.

\medskip

However, maximizing the average rank or searching for Pareto-optimal solutions may not be  satisfactory as it can lead to particularly unfair allocations. 
For this reason, many fairness criteria have been introduced. We will focus here on maximizing the \emph{minimum rank} ($mrk$) of the resources held by the agents, defined as:
$$mrk(\pi) = \min_{a_i \in \agentSet} rank_{a_i}(\pi_{a_i}).$$

Again, if we were to interpret the rank as a cardinal utility function, the minimum rank would be equivalent to the egalitarian welfare. Maximizing the minimum rank follows Rawls' principle of maximizing the welfare of the worst-off \citep{rawls1971theory}. It has been introduced by \cite{pazner1978egalitarian} and is now a very common rule in fair division \citep{thomson1983problems,sprumont1996axiomatizing,nguyen2014computational}. 

\subsection{Properties of the procedures}
\label{subsec:proceduresQuality}

Let us now turn our attention to the allocation procedures. An allocation procedure $\proc$ is a mapping taking as input an instance $I$ and returning an allocation $\proc(I)$.

We will say that an allocation procedure is Pareto-optimal if it always return a Pareto-optimal allocation. Given that the agents initially own the resources in our setting, another very common efficiency requirement is that of \emph{individual rationality}. It stipulates that no agent should be worse-off in the final allocation.

\begin{definition}[Individual rationality]
	A procedure $\proc$ is said to be individually rational if for every instance $I = \tuple{\agentSet, \resourceSet, L, \pi^0}$ and every agent $a_i \in \agentSet$, we have:
	$$\proc(I) \succeq_i \pi^0_{a_i}$$
\end{definition}

Let us illustrate through an example the different criteria introduced so far.

\begin{example}
	\label{ex:firstEx}
	
	Let us consider the following instance with 5 agents and 5 resources. The preferences presented below are single-peaked with respect to $r_1 \lhd r_2 \lhd r_3 \lhd r_4 \lhd r_5$. The initial allocation  $\pi^0 = \tuple{r_1, r_2, r_3, r_4, r_5}$ is represented by the underlined resources.	
	\begin{align*}
	a_1: \enspace & r_3 \succ_{1} r_4 \succ_{1} r_5 \succ_{1} \held{r_2} \succ_{1} \heldc{r_1} \\
	a_2: \enspace & r_3 \succ_{2} \held{r_4} \succ_{2} r_5\succ_{2} \heldc{r_2} \succ_{2} r_1 \\
	a_3: \enspace & r_4 \succ_{3} \held{r_5} \succ_{3} \heldc{r_3} \succ_{3} r_2 \succ_{3} r_1 \\
	a_4: \enspace & \held{r_3} \succ_{4} \heldc{r_4} \succ_{4} r_5 \succ_{4} r_2 \succ_{4} r_1 \\
	a_5: \enspace & \held{r_1} \succ_{5} r_2 \succ_{5} r_3 \succ_{5} r_4 \succ_{5} \heldc{r_5} 	
	\end{align*}
	
	The allocation $\pi^0$ is not Pareto-optimal as it is Pareto-dominated by the squared allocation $\held{ \pi} = \tuple{r_2, r_4, r_5, r_3, r_1}$. We have $ark(\held{\pi}) = 4$ and $mrk(\held{\pi}) = 2$.
\end{example}

We conclude this section by introducing another desirable property that is \emph{strategy-proofness}. Informally, an allocation procedure is strategy-proof if no agent could get a strictly better outcome by lying instead of revealing her true preferences.

\begin{definition}
	Let $L = \tuple{\succ_j}_{a_j \in \agentSet}$ be a preference profile. For a given agent $a_i \in \agentSet$, an $i$-variant $L_{-i} = \tuple{\succ'_j}_{a_j \in \agentSet}$ of $L$ is a preference profile such that:
	$$\forall a_j \in \agentSet \backslash \{a_i\}, \succ_j = \succ'_j \text{ and  } \succ_i \neq \succ'_i.$$
	
	An allocation procedure $\proc$ is said to be strategy-proof if for every instance $I = \tuple{\agentSet, \resourceSet, L, \pi^0}$, every agent $a_i \in \agentSet$, every $i$-variant $L_{-i}$ of $L$, we have:
	$$\proc(\tuple{\agentSet, \resourceSet, L, \pi^0}) \succeq_i \proc(\tuple{\agentSet, \resourceSet, L_{-i}, \pi^0}).$$
\end{definition}

\section{Centralized allocation procedures for housing markets}
\label{sec:centralProcedures}

This section introduces two centralized allocation procedures that will be used as references in the paper. The first one is the TTC algorithm \citep{shapley1974cores}, which can be used without any domain restriction. It is well known to satisfy the three main desirable properties of an allocation procedure: \emph{Pareto-optimality}, \emph{individual rationality} and \emph{strategy-proofness}.
The second one,  Crawler \citep{bade2019matching}, is specially designed for single-peaked domains. It satisfies the same properties as TTC.
Both of theses procedures are based on the notion of deals.

\subsection{Improving deals}
\label{subsec:deals}

A deal is a vector of agents, usually denoted by $\mu = \tuple{a_1, \ldots, a_k}$, where $a_i \in \agentSet, \forall i \in \{1, \ldots, k \}$. It represents an exchange where agent $a_i$ gives her resource to agent $a_{i + 1}$ for each $i \in \{1, \ldots, k-1 \}$ and agent $a_k$ gives her resource to agent $a_1$. With a slight abuse of notations, given a deal  $\mu = \tuple{a_1, \ldots, a_k}$, agent   $a_{k+1}$ will refer to  agent  $a_1$ and agent $a_0$ (i.e. the agent before $a_1$) to agent $a_k$.  For the particular case of deals involving only two agents, i.e. $k = 2$, we will talk about \emph{swap-deals}.

Interestingly, it can be observed that in housing markets, any reallocation (permutation of resources) can be implemented as a collection of disjoint cycle-deals \citep{shapley1974cores}. This notion of deal is thus sufficient in this context.

Let us now introduce the concept of improving deals.

\begin{definition}
	Let $\pi$ be an allocation, and $\mu = \tuple{a_1, \ldots, a_k}$ a deal involving $k \geq 1$ agents. The allocation $\pi[\mu]$ obtained by applying the deal $\mu$ to $\pi$ is defined by: 
	\[\begin{cases}
	\pi[\mu]_{a_i} = \pi_{a_{i - 1}} &\mbox{ if } i \in \{1, \ldots, k \}, \\
	\pi[\mu]_{a_i} = \pi_{a_i} & \mbox{ otherwise}.
	\end{cases}\]
	
	A deal is said to be \emph{improving} if $\pi[\mu] \succ_i \pi$ for every agent $a_i$ involved in $\mu$.
\end{definition}

\noindent Observe that for a deal $\mu$ of length 1, we have $\pi[\mu] = \pi$. This trivial case consists of an agent giving her resource to herself. It will be useful to simplify the presentation of TTC and  Crawler. In particular, when an agent already holds her top resource, she should also hold it in the final allocation because of individual rationality.

It is also straightforward to see that a procedure applying only improving deals trivially satisfies individual rationality.

\subsection{Gale's Top Trading Cycle algorithm}

The TTC algorithm takes as input an instance $I = \tuple{\agentSet, \resourceSet, L, \pi^0}$ and proceeds as follows. The algorithm maintains a set of available agents $N$ and a set of available resources $R$ where initially $R = \mathcal{R}$ and $N = \mathcal{N}$. At each step of the algorithm,  a directed bipartite graph $G = \tuple{V, E}$, with $V = N \cup R$, is defined. The nodes of $G$ represent the agents in $N$ and  the resources in $R$, and the set of edges $E$ is such that:
\begin{itemize}
	\item there is a directed edge $(a_i, r_i)$ between $a_i$ and $r_i$ if and only if $r_i = top_R(a_i)$ i.e. agents are linked to their preferred  resource in $R$, 
	
	\item there is a directed edge $(r_i, a_i)$ between $r_i$ and $a_i$ if and only if $r_i = \pi^0_{a_i}$ i.e. resources are linked to their owner in $\pi^0$.
\end{itemize}
Note that there always exists at least one cycle in $G$ and that cycles correspond to improving deals. The cycle-deals constructed can be of size 1 if an agent already owns her top resource in $R$.
The TTC algorithm selects one of the cycles $\mu$ in $G$. The agents and resources involved in $\mu$ are then removed from $N$ and $R$. A new graph $G'$ is computed with the remaining agents and resources. The process is iterated on the new graph $G'$ and $\pi^0[\mu]$ until an empty graph has been reached. 
Note that the outcome of TTC is unique (regardless of the possibly different choices of cycles to implement during the process).

\revisedC{
\begin{algorithm}[t]
	\DontPrintSemicolon
	\KwIn{An instance $I = \tuple{\agentSet, \resourceSet, L, \pi^0}$}
	\KwOut{An allocation $\pi$}
	$\pi \gets$ empty allocation \;
	$R \gets \resourceSet$: list of resources\;
	$N \gets \agentSet$: list of agents\;

	\While {$N \neq \emptyset$} {
		$E \leftarrow \emptyset$\;
		$V  \leftarrow N \cup R$\;
		\For {each agent $a \in N$} {
			Add a directed edge in $E$ between $a$ and $top_R(a)$
			}
		\For{each resource $r \in R$}{
			Add a directed edge in $E$ between $r$ and her owner in $\pi^0$
		}

		Select a cycle $\mu$ from the graph $G = \tuple{V, E} $ \;
		$do(\mu, \pi)$\;
	}
	\Return $\pi$
	\caption{ TTC algorithm}
	\label{algo:TTC}
\end{algorithm}
}

\hidden{
	
	\begin{figure}
		\centering
		\begin{tikzpicture}		
		\tikzstyle{arg-out}=[circle,draw, line width=1pt]
		
		\node (G) at (-0.2, 0.75) [] {$G:$};
		
		\node (A1) at (1,1.5) [arg-out,label=above:\footnotesize{$a_1$}] {};
		\node (A2) at (2,1.5) [arg-out,label=above:\footnotesize{$a_2$}] {};
		\node (A3) at (3,1.5) [arg-out,label=above:\footnotesize{$a_3$}] {};
		\node (A4) at (4,1.5) [arg-out,label=above:\footnotesize{$a_4$}] {};
		\node (A5) at (5,1.5) [arg-out,label=above:\footnotesize{$a_5$}] {};

		\node (R1) at (2,0) [arg-out,label=below:\footnotesize{$r_1$}] {};
		\node (R2) at (5,0) [arg-out,label=below:\footnotesize{$r_2$}] {};
		\node (R3) at (3,0) [arg-out,label=below:\footnotesize{$r_3$}] {};
		\node (R4) at (4,0) [arg-out,label=below:\footnotesize{$r_4$}] {};
		\node (R5) at (1,0) [arg-out,label=below:\footnotesize{$r_5$}] {};

		\draw [->, very thick] (A1) -- (R1);
		\draw [->, very thick] (A2) -- (R5);
		\draw [->, very thick] (A3) -- (R3);
		\draw [->, very thick] (A4) -- (R4);
		\draw [->, very thick] (A5) -- (R4);
		
		\draw [->,very thick] (R5) -- (A1);
		\draw [->, very thick] (R2) -- (A5);
		\draw [->, very thick] (R1) -- (A2);
		\draw [->, very thick] (R4) -- (A4);
		\draw [->, very thick] (R3) -- (A3);
		
		\end{tikzpicture}
		\caption{Bipartite graph created by the TTC algorithm for $\pi^0$ as defined in Example \ref{ex:firstEx}}
		\label{fig:exTTC}
	\end{figure}
}

\revisedC{
A formal description of  TTC procedure is given in Algorithm \ref{algo:TTC}. Note that we make use of the sub-procedure $do(\mu, \pi)$ which simply implements the cycle-deal $\mu$ and adds the resulting allocation (restricted to the agents involved in the deal)  in the allocation $\pi$. It then removes the agents and the resources involved in the deal $\mu$ from the lists of available agents and resources, $N$ and $R$ respectively.
}

\begin{example}
	\label{ex:TTC}
	Let us consider the instance defined in Example \ref{ex:firstEx}. 
	The reader can check that the allocation returned by TTC is $\pi^{TTC} = \tuple{r_5, r_2, r_4, r_3, r_1}$. 
	Indeed, the first cycle-deal that can be applied is $\mu_1 = \tuple{a_3, a_4}$. The next one is $\mu_2 = \tuple{a_1, a_5}$. Finally, $a_2$ remains with her initial resource which corresponds to the cycle $\mu_3 = \tuple{a_2}$. 
\end{example}

\subsection{Crawler algorithm}

In Crawler algorithm, agents are initially ordered along the single-peaked axis according to the resource they hold. The first agent is then the one holding the resource on the left end side of the axis. As in TTC, the list of available resources is denoted by $R$ and is ordered according to the single-peaked axis. A list of available agents $N$ such that the $i^{th}$ agent of the list is the one who holds the $i^{th}$ resource in $R$ is also maintained. 

Considering agents sequentially from the first in $N$ to the last in $N$, Crawler checks for each agent $a_i$ where her top resource, $top_R(a_i)$, is on the axis.\footnote{Note that the algorithm can equivalently be executed from the last agent to the first one.}

	\begin{itemize}
		\item  If $top_R(a_i)$ is on her right, the algorithm moves to the next agent. 
		\item  If $a_i$ holds her top resource $top_R(a_i)$, then $top_R(a_i)$ is allocated to $a_i$. Agent $a_i$ and the resource $top_R(a_i)$ are removed from $N$ and $R$. The algorithm starts again from the agent on the left end of the axis. 
		\item If $top_R(a_i)$ is on the left of $a_i$, the agent is allocated her top resource $top_R(a_i)$. We denote by $t^*$ the index of $top_R(a_i)$ and $t$ the index of $a_i$ (we have $t^* < t$). Then, every agent between $t^*$ and $t - 1$ receives the resource held by the agent on her right (the resources ``crawl'' towards left). Once again, $a_i$ and $top_R(a_i)$ are removed from $N$ and $R$ and the algorithm restarts from the first agent.
\end{itemize}
Once all the resources have been allocated, the algorithm terminates. 

\revisedC{
A formal description of Crawler procedure is given in Algorithm \ref{algo:Crawler}. Note that we make use of the sub-procedure $pick(a_{t^*},r,N,R,\pi)$ which simply assigns the resource $r$ to the given agent $a_{t^*}$  in the allocation $\pi$, and then removes the agent and the resource from the lists of available agents and resources, $N$ and $R$ respectively. Since the list of resources is ordered following the single-peaked axis and the $i^{th}$ agent in $N$ corresponds to the owner of the  $i^{th}$ resource in $R$, the removal of $r$ and $a_{t^*}$  is in fact equivalent to assigning $r$ to agent $a_{t^*}$ and crawling the resources from right to left.

\begin{algorithm}[t]
	\DontPrintSemicolon
	\KwIn{An instance $I = \tuple{\agentSet, \resourceSet, L, \pi^0}$ single-peaked with respect to $\lhd$}
	\KwOut{An allocation $\pi$}
	$\pi \gets$ empty allocation \;
	$R \gets \resourceSet$: list of resources sorted accordingly to $\lhd$\;
	$N \gets \agentSet$: list of agents such that the $i^{th}$ agent is the one who initially holds the $i^{th}$ resource in $R$\;
	\While {$N \neq \emptyset$} {
		$t^* \gets |N|$\;
		\For {$i = 0$ to $|N| - 1$} {
			\If {$r_i \succ_{i} r_{i + 1}$} {
				$t^* \gets i$\;
				Break\;
			}
		}
		$r \gets top_R(a_{t^*})$ \;
		$pick(a_{t^*},r,N,R, \pi)$\;
	}
	\Return $\pi$
	\caption{Crawler algorithm}
	\label{algo:Crawler}
\end{algorithm}
}

\medskip

Let us illustrate the execution of  Crawler: 

\begin{example}
	On the instance of Example \ref{ex:firstEx}, agent $a_4$ is the first agent whose top resource is not on her right, she thus receives her top resource $r_3$. The second step matches agent $a_3$ to $r_4$. On the third step, agents $a_1$ and $a_2$ both  have their top resources (among the remaining resources) on the right but the last agent $a_5$ has hers on her left. $a_5$ is then  matched to $r_1$.  Resource $r_2$ crawls to agent  $a_1$  and resource $r_5$ crawls to agent $a_2$. On the fourth step, $a_2$ picks her current resource $r_5$. Finally, $a_1$ is assigned resource $r_2$.
	The allocation returned by Crawler is $\pi^{C} = \tuple{r_2, r_5, r_4, r_3, r_1}$. 
	At each step $i$ of the procedure, an improving cycle-deal $\mu_i$ is applied (with the last agent in the cycle picking her top resource among the remaining ones): $\mu_1 = \tuple{a_3,a_4}$, $\mu_2 = \tuple{a_3}$,  $\mu_3 = \tuple{a_1,a_2,a_5}$, $\mu_4 = \tuple{a_2}$, $\mu_4 = \tuple{a_1}$. 
	One can observe that on this example the allocation returned by Crawler is not the same as the one returned by TTC. However, both procedures lead to the same minimum rank $mrk$ and average rank $ark$.
\end{example}
Interestingly, a variant of this procedure allows to check Pareto-optimality in single-peaked domains in linear time \citep{beynier2020optimal}.

\section{Swap dynamics: a family of procedures based on swap-deals}
\label{sec:decentralProcedures}

We will now focus on dynamics based on local exchanges between the agents. Let us first introduce some additional definitions.
	
	Given an allocation $\pi$, we denote by $C_k(\pi), k \geq 2,$ the set of all the improving deals of size at most $k$ that can be applied from $\pi$:
	$$C_k(\pi) = \{\mu \mid \mu \text{ is an improving deal and } |\mu| \leq k \}.$$
	
	When investigating procedures based on improving exchanges, we will try to reach allocations that are stable with respect to some deals.
	
	\begin{definition}
		For a given $k \in \{1, \ldots, n \}$, an allocation $\pi$ is said to be stable with respect to $C_k$ if $C_k(\pi) = \emptyset$.
	\end{definition}

	It is obvious from this definition that if an allocation is stable with respect to $C_k$ for a given $k$, it is also stable with respect to any $k' < k$. Moreover, since we are considering housing markets, it can be observed that an allocation is Pareto-optimal if and only if it is stable with respect to $C_n$.

\medskip

Dynamics based on local exchanges start from an initial allocation and let the agents negotiate improving cycle-deals involving at most $k$ agents until reaching an allocation stable with respect to $C_k$. 

	Following this process, to each $k \in \{1, \ldots, n\}$ corresponds a family of allocation procedures based on $C_k$. Indeed, at each step, the improving deal to be implemented can be selected in many different ways. We call \emph{selection heuristic} a mapping giving the deal to implement at a given step of the procedure. A dynamic based on local exchanges will thus be defined by a given $k$ and a specific selection heuristic. 
	Depending on the selection heuristic, it may be the case that the selected deal is not an improving one. Such a deal could then be refused by the agents. We will say that a deal is \emph{successful} if all agents involved in it  agree to exchange. The sequence of proposed deals, together with the fact they were successful or not will be called a \emph{history}.
	
\revisedC{
\begin{algorithm}[t]
	\DontPrintSemicolon
	\KwIn{An instance $I = \tuple{\agentSet, \resourceSet, L, \pi^0}$, \;  a maximum size $k$ for the deals, \; a selection heuristic $\sigma$}
	\KwOut{An allocation $\pi$ stable with respect to $C_k$}
	$\pi \gets \pi^0$ \;

	\While {$\pi$ is not stable with respect to $C_k$} {
		$\mu \leftarrow \sigma(\agentSet, k, \dots) $ \;

		\If{$\mu$ is successful} {
				$do(\mu, \pi)$\;
			}
	
	}
	\Return $\pi$
	\caption{Cycle-deals dynamics}
	\label{algo:swaps}
\end{algorithm}

A general formal description of cycle-deals dynamics is given in Algorithm \ref{algo:swaps}.  The $do(\mu, \pi)$ method is similar to the one described in Algorithm \ref{algo:TTC} and implements the deal $\mu$ (as chosen by the selection heuristic $\sigma$) on the allocation $\pi$. Note that the selection heuristic itself takes as input at least the set of agents $\agentSet$ and the maximum size $k$ of the cycle-deals (but it may be more informed and take more parameters, as we shall discuss later in this section). Unless stated otherwise we shall assume from now on that $k=2$ and omit this parameter, as we will mainly be interested in deals of size 2. We will also use \emph{swap-stability} to refer to stability with respect to $C_2$.

}

We say that the procedure has \emph{reached} the allocation obtained upon termination. 
Termination occurs when, for every possible swap-deal $\tuple{a_x, a_y}$, there exists a latest unsuccessful proposal, such that there was subsequently no successful swap-deal involving $a_x$ nor $a_y$. This guarantees in particular that the process cannot end as long as there remains an improving deal that has not yet been proposed. 
To ensure termination, the selection heuristic should not prevent some deals from happening. We will thus require heuristics to satisfy a property of full coverage as defined below.

\begin{definition}
	A selection heuristic $\sigma$ has \emph{minimal}  (resp. \emph{full}) \emph{coverage} if for any $a_x, a_y \in \agentSet^2$,  there exists at least one round in the sequence (resp. after the latest successful swap-deal involving $a_x$ or $a_y$ if there is one), when the swap-deal $\tuple{a_x, a_y}$  is proposed. 
\end{definition}

	$\swapfamily$ will denote the family of swap dynamics defined with respect to a selection heuristic with full coverage. 
	It is useful to make a further distinction between different types of selection  heuristics, depending on the information they take as input. In particular, heuristics may require preferential information (e.g. which resources agents would be happy to swap their current resource with), or on the contrary, be solely based on observable information (e.g. the history of deals). We will mostly focus on the latter in what follows.

	 A history-based selection heuristic $\sigma$ is a function taking as input the set of agents $\agentSet$, the history of deals $h$, and returning a  deal $\sigma(\agentSet, h)$ to be proposed. Swap dynamics which rely on history-based selection heuristics will be called \emph{history-based swap dynamics}.

\medskip

A swap dynamics equipped with a given history-based selection heuristic $\sigma$ on the instance $I = \tuple{\agentSet, \resourceSet, L, \pi^0}$, proceeds iteratively as follows. The history $h$ is initially empty. At each round corresponding to an allocation $\pi$ and a history $h$, the allocation is updated to $\pi'$ defined as:
\[\pi' = \begin{cases}
\pi[\sigma(\agentSet, h)] & \text{if the agents in } \sigma(\agentSet, h) \text{ agree on swapping } \\
\pi & \text{otherwise}.
\end{cases}\]
The deal $\sigma(\agentSet, h)$ together with its ``success status'' are then added to the history $h$.


We now give some examples of selection heuristics which will be studied in this paper. We start with \emph{round-robin} heuristics, which simply specify a predefined way to order the different pairs of agents (i.e. possible swaps), and repeat it until termination. There are several ways to proceed, we give here two prominent examples:

\begin{itemize}
	
	\item \emph{Round-Robin over the Agents} ($M_2$-RRA): agents are ordered and paired following their name ($\tuple{a_1, \ldots, a_n}$) in a round-robin fashion. The first agent $a_1$ is paired with each other agent, the second agent $a_2$ is paired with each other agent $a_j$ with $j > 2$ and so on. The agents then go over possible deals by iterating over the following sequence: 
	\[
	(a_1, a_2), (a_1, a_3), \ldots, (a_1, a_n), (a_2, a_3),\ldots, (a_{n-1}, a_n), (a_1, a_2), \ldots 
	\]
	\item \emph{Round-Robin over Pairs of agents}  ($M_2$-RRP): agents are ordered and paired following their name ($\tuple{a_1, \ldots, a_n}$) in a round-robin fashion. In this case, the first agent is paired with the second agent, the second agent is paired with the third... Hence, the agents go over possible deals by iterating over the following sequence: 
	\[
	(a_1, a_2), (a_3, a_4),  \ldots, (a_{n-1},a_n),(a_1, a_3), (a_2,a_4), \ldots, (a_1,a_n), \ldots
	\]
\end{itemize}

There is an obvious bias in the way deals are selected, it is thus natural to introduce some stochasticity in the process. One way to do this is as follows: 

\begin{itemize}	
	\item \emph{Randomized Round-Robin over deals} ($M_2$-RRR): the heuristic picks uniformly at random a permutation among all the possible deals. This permutation defines a round-robin order in which the deals are considered. The agents then go over possible deals by iterating over this permutation. 
\end{itemize}

The three previous heuristics all guarantee by construction that all possible swap-deals were proposed before a swap-deal gets proposed again. The following natural heuristic do not have such guarantee. 

\begin{itemize}
	\item \emph{Uniform} ($M_2$-U): a pair of agents (i.e. a swap-deal)  $(a_i,a_j)$ is selected uniformly at random among all possible pairs. 
	\item \emph{Random matchings} ($M_2$-RM): 
	this heuristic proceeds in succession of matching steps where at each step, every agent is matched to a unique other agent and all the resulting swap-deals are proposed to the agents. This simulates a natural market where agents are paired randomly and simultaneously try to perform bilateral deals. 
\end{itemize}

Note that these selection heuristics do not all have the same degree of decentralization. In particular, $M_2$-U can be easily executed in a fully decentralized way. Round-robin heuristics require a central entity to broadcast the sequence of pairs. Then, the agents can meet in a distributed way. Finally, $M_2$-RM can be implemented in a distributed way using a protocol ensuring that an agent can not encounter several agents simultaneously. Each agent $a_i$ then selects uniformly at random an agent $a_j$ to encounter and contact her. If agent $a_j$ is already engaged in another encounter, $a_i$ selects another agent.

\medskip

We conclude the section by a straightforward observation: different selection heuristics can lead to different outcomes given the same initial allocation. Let us illustrate this with a simple example.

\begin{example}
	Consider the instance described in Example \ref{ex:firstEx}. We showed that Crawler and TTC return different allocations. In fact, these allocations can be reached by a sequence of swap-deals. 
	Observe first that the cycle-deals applied by TTC are at most of length 2 (Example \ref{ex:TTC}). The allocation $\pi^{TTC} = \tuple{r_5, r_2, r_4, r_3, r_1}$ is obtained through two swap-deals: $\tuple{a_3, a_4}$ and $\tuple{a_1, a_5}$. 
	The allocation $\pi^C = \tuple{r_2, r_5, r_4, r_3, r_1}$ returned by Crawler also is reachable by swap-deals. It is obtained by applying the following sequence: $\tuple{a_3, a_4}$, $\tuple{a_2, a_5}$ and $\tuple{a_5, a_1}$. 
	Since these two allocations are Pareto-optimal, they are swap-stable. This shows that the way swap-deals are selected may affect the final stable allocation. Notice that these are not the only two swap-stable allocations reachable from $\pi^0$ of the instance: $\tuple{r_3,r_2,r_5,r_4,r_1}$ and $\tuple{r_2,r_3,r_5,r_4,r_1}$ are the two other ones.
\end{example}

This example suggests that the allocations returned by Crawler and by TTC can both be reached via swap-deals ---a point we will make formal in the next section. 
More generally, as TTC and  Crawler both provide desirable guarantees, on Pareto-optimality, individual rationality and strategy-proofness, it is natural to study whether swap dynamics enjoy similar properties.

\section{Properties of swap dynamics}
\label{sec:swapEfficiency}

We now investigate  properties of swap dynamics. 
Recall first the easy observation made in Section \ref{sec:centralProcedures}: swap dynamics, because they rely on improving deals, are individually rational. But what about strategy-proofness and Pareto-optimality? 
We first discuss strategy-proofness, showing that such procedures are in general subject to manipulation.
We next prove that any allocation stable with respect to swap-deals is Pareto-optimal and that both the allocation returned by TTC and  Crawler can be reached via swap-deals. Finally, we show that the single-peaked domain is a maximal domain when it comes to guaranteeing Pareto-optimality.

\subsection{Strategy-proofness}

Although strategy-proofness is usually defined for centralized procedures, the question is still relevant for swap dynamics. 

In a swap dynamic, an agent has the opportunity to behave strategically only when she is asked to accept or reject a proposed deal. 
	As usual, we assume here that a potential manipulator has full knowledge of the preferences of the other agents and is aware of the fact that the selection heuristics has minimal coverage. In other words, the manipulator only needs to know whether she will ever have the opportunity to swap with some other agent or not.
A swap dynamic will be called \emph{strategy-proof} when, for every instance, at no point during the procedure an agent can be better off by accepting a non-improving deal or by refusing an improving one. Otherwise, the procedure is \emph{manipulable}. 

\begin{proposition}
	\label{prop:no-strategyproof}
	Any history-based swap dynamic with minimal coverage is manipulable.
\end{proposition}

\begin{proof} 
	We are given an arbitrary history of deals $h$ starting with $(a_x,a_y)$, as produced by the  selection heuristic.  
	We are now going to show that we can build an instance $I = \langle \agentSet, \resourceSet, L, \pi^0 \rangle$ such that agent $a_x$ will have an incentive to accept the first swap-deal $(a_x,a_y)$, even though this is not rational.  
	The instance involves $a_x$, $a_y$, $a_z$, as well as $n-3$ other dummy agents, and assumes preferences to be single-peaked with respect to $r_1 \lhd r_2 \lhd r_3 \lhd r_{d_1} \lhd \dots \lhd r_{d_{n-3}}$. 
	\begin{align*}
	a_x : \enspace & r_2 \succ_{1} \heldc{r_3} \succ_{1} r_1 \succ_1 r_{d_1} \dots \\
	a_y : \enspace & r_2 \succ_{2} r_3 \succ_{2} \heldc{r_1} \succ_2 r_{d_1} \dots  \\
	a_z : \enspace & r_1 \succ_{3} \heldc{r_2} \succ_{3} r_3 \succ_3 r_{d_1} \dots  \\
	d_1: \enspace & \heldc{r_{d_1}} \succ \dots \\
	\vdots \enspace & \\
	d_{n-3}:   \enspace & \heldc{r_{d_{n-3}}} \succ \dots \\
	\end{align*} 
	Observe first that all the dummy agents have their top resource from the start, hence they will not be involved in any deal. 
	Now, consider the situation when $a_x$ is given the opportunity to deal with $a_y$. 
	If $a_x$ is truthful (as all the other agents), no improving swap-deal involving agent $a_x$ can occur, and she will end up with resource $r_3$. 
	Now, suppose instead that $a_x$ strategically accepts $\tuple{a_x, a_y}$.
	This deal is improving for $a_y$ so if $a_x$ agrees on it, it will be implemented. In that case, $a_y$ can no longer exchange with $a_z$, hence the only improving swap-deal left is 
	$\tuple{a_x, a_z}$. 
	Because the selection heuristic has minimal coverage, this opportunity will occur at some point in the future. 	
	In the final allocation, agent $a_x$ would then hold her top resource $r_2$ (obtained from the swap-deal 	$\tuple{a_x, a_z}$). 
	\qed
\end{proof}

Now, suppose that for a given selection heuristic, there exists a swap-deal $\tuple{a_x, a_y}$ that will never be proposed. Consider then the instance where all the agents have their top resource, except for $a_x$ and $a_y$ who would be happy to perform a swap-deal---but will never get a chance to. As no other swap-deal is possible, the outcome is certainly not Pareto-optimal. 
This leads to the following observation.

\begin{observation} \label{obs:subSet}
	Consider a history-based swap dynamic $M_2$ defined with respect to a selection heuristic $\sigma$. If $M_2$ is Pareto-optimal, then $\sigma$ has minimal coverage.
\end{observation}

Together with Proposition \ref{prop:no-strategyproof}, this tells us that there is a fundamental tension between strategy-proofness and Pareto-optimality for history-based swap-deal procedure. 

\begin{proposition}
	No history-based swap dynamic can be both Pareto-optimal and strategy-proof.
\end{proposition}

The reason why other types of dynamics may not be concerned by this result is that they can potentially condition the selection of deals to the preferences.

\subsection{Pareto-optimality of swap dynamics}

We show here that any allocation reached by swap dynamics is Pareto-optimal.

\begin{theorem} \label{theo:PO}
	Let $I$ be a single-peaked instance, every allocation $\pi$ in $I$ that is stable with respect to swap-deals also is stable with respect to $C_n$.
\end{theorem}

\begin{proof}
	
		Let us consider an allocation $\pi$ stable with respect to $C_2$ but not with respect to $C_n$. Since $\pi$ is not stable with respect to $C_n$, there exists at least one improving deal in $\pi$. Let $\mu$ be the smallest improving cycle-deal, i.e., which involves the smallest number of agents. Assume without loss of generality that  $\mu = \tuple{a_1, \ldots, a_k}$, with $2 < k \leq n$.
	
	Since $\mu$ is an improving deal, every agent in $\mu$ is happy to exchange with the agent coming before her in $\mu$:
	\begin{align}
	\label{eq:proofSwapPOLine1}
	\pi_{a_{i - 1}} \succ_{i} \pi_{a_i}, \forall a_i \in \mu.
	\end{align}
	
	Moreover, as there exists no improving swap-deal in $\pi$ ($C_2(\pi) = \emptyset$) agents do not want to exchange with the agent coming after them in $\mu$:
	\begin{align}
	\label{eq:proofSwapPOLine2}
	\pi_{a_i} \succ_i \pi_{a_{i + 1}}, \forall a_i \in \mu.
	\end{align}
	Indeed, because of \eqref{eq:proofSwapPOLine1}, an improving swap-deal would otherwise exist in $\pi$.
	
	We now show by induction on the size of $\mu$, denoted by $k$, that such an improving deal cannot exist.
	
	If $k = 3$, consider without loss of generality that $\mu = \tuple{a_1, a_2, a_3}$. From \eqref{eq:proofSwapPOLine1} and \eqref{eq:proofSwapPOLine2}, we obtain the following profile where underlined resources indicate the initial allocation:
		$$\begin{array}{cl}
		a_1: & \pi_{a_3} \succ_1 \heldc{\pi_{a_1}} \succ_1 \pi_{a_2}, \\
		a_2: & \pi_{a_1} \succ_2 \heldc{\pi_{a_2}} \succ_2 \pi_{a_3}, \\
		a_3: & \pi_{a_2} \succ_3 \heldc{\pi_{a_3}} \succ_3 \pi_{a_1}.
		\end{array}$$
	The triplet of resources $\tuple{\pi_{a_1}, \pi_{a_2}, \pi_{a_3}}$ is thus a witness of the violation of the worst-restrictedness condition for a profile to be single-peaked (Proposition \ref{prop:SP=>WR}). Indeed, all the three resources are ranked last by an agent when we restrict our attention to these resources. The contradiction is thus established.
	
	Suppose now that $\pi$ is stable with respect to $C_{k - 1}$. We will show that no improving deal of size $k$ exists in $\pi$. From the induction hypothesis, we get that:
	\begin{align}
	\label{eq:proofSwapPOLine3}
	\pi_{a_i} \succ_i \pi_{a_j}, \forall a_i, a_j \in \mu, a_j \not \in \{a_{i - 1}, a_i\}.
	\end{align}
	Indeed there would otherwise exist two agents $a_l$ and $a_l'$, that are not next to one another in $\mu$, such that $\pi_{a_l'} \succ_l \pi_{a_l}$. It would then have been possible to ``cut'' $\mu$ between those two agents so that $a_l$ receives $\pi_{a_l'}$. The new cycle-deal obtained would also have been improving and then an improving deal of size strictly smaller than $k$ would exist.
	
	Because the profile is single-peaked, it is also worst-restricted (Proposition \ref{prop:SP=>WR}) and there exist thus at most two resources ranked last by an agent among the ones appearing in $\mu$. Call $\pi_{a_w}$ one such resource holds by agent $a_w$, and consider the triplet of resources $R = \tuple{\pi_{a_{w - 1}}, \pi_{a_{w}}, \pi_{a_{w + 1}}}$. From \eqref{eq:proofSwapPOLine1}, \eqref{eq:proofSwapPOLine2} and \eqref{eq:proofSwapPOLine3} we get:
	$$\begin{array}{cl}
	a_w: & \pi_{a_{w - 1}} \succ_w \heldc{\pi_{a_w}} \succ_w \pi_{a_{w + 1}}, \\
	a_{w + 1}: & \pi_{a_{w}} \succ_{w + 1} \heldc{\pi_{a_{w + 1}}} \succ_{w + 1} \pi_{a_{w - 1}}.
	\end{array}$$
	Hence when restricting preferences to $R$, for every resource in $R$, there exists an agent ranking it last among the resources in $R$. This violates the worst-restrictedness condition for the single-peaked profile and sets the contradiction.
	
	Overall we have proved that no improving deal exists in $\pi$ which entails that $\pi$ is stable with respect to $C_n$. $\qed$
\end{proof}

This theorem states that the $C_k$-stability hierarchy collapses at the $C_2$ level in single-peaked housing markets. Remember that every Pareto-optimal allocation is stable with respect to $C_n$. This result provides then a new characterization of Pareto-optimality in our setting.

\begin{corollary}
	In a single-peaked housing market, an allocation $\pi$ is Pareto-optimal if and only if it is stable with respect to $C_2$.
\end{corollary}

Stating this result in terms of stability with respect to $C_n$ and not just Pareto-optimality gives us more flexibility. Indeed, in the more general setting where there are more resources to allocate than the number of agents, the result of Theorem \ref{theo:PO} still holds \citep{beynier2018efficiency} but it is no longer the case that Pareto-optimality implies $C_n$-stability.

\medskip

As we have proven that the allocation reached by swap-deals is Pareto-optimal, a natural question is then whether \emph{every} allocation that Pareto-dominates the initial allocation can be reached by swap-deals. It is not the case. 

\begin{proposition}
	\label{prop:notallPOreachable}
	There exists an instance $I = \tuple{\agentSet, \resourceSet, L, \pi^0}$ for which there is an allocation $\pi$ that Pareto-dominates $\pi^0$ and that can not be reached by a sequence of improving swap-deals.
\end{proposition}

\begin{proof}
	Let us consider the following instance with three agents where the initial allocation $\pi^0 = \tuple{r_3, r_2, r_1}$ is the one underlined.
	$$\begin{array}{cccccc}
	a_1: & r_1 & \succ_1 & \held{r_2} & \succ_1 & \heldc{r_3} \\
	a_2: & \held{r_1} & \succ_2 & \heldc{r_2} & \succ_2 & r_3 \\
	a_3: & r_2 & \succ_3 & \held{r_3} & \succ_3 & \heldc{r_1}
	\end{array}$$
	The squared allocation $\held{\pi} = \tuple{r_2, r_1, r_3}$ is Pareto-optimal. However from $\pi^0$ only two deals are possible: $\mu_1 = \tuple{a_1, a_3}$ that reaches allocation $\pi'=\tuple{r_1, r_2, r_3}$, or $\mu_2 = \tuple{a_2, a_3}$ that leads to $\pi''=\tuple{r_3, r_1, r_2}$. No sequence of improving swap-deals can thus reach $\pi$. 
\end{proof}

It is however interesting to note that both the allocation returned by TTC and by  Crawler can always be reached through improving swap-deals.

\begin{proposition}
	\label{prop:ttc-is-swap-reachable}
	Let $I = \tuple{\agentSet, \resourceSet, L, \pi^0}$ be an instance and let $\pi^{TTC}$ be the allocation returned by the TTC algorithm on $I$. Then $\pi^{TTC}$ is reachable by swap-deals from $\pi^0$. 
\end{proposition}

\begin{proof} 
	We show that any cycle-deal applied  by TTC can be implemented as a sequence of improving swap-deals. For a given step of the algorithm, we consider $R$ the set of resources and $N$ the set of agents remaining at this step and $\pi$ the current partial allocation.
	
	Let $\mu$ be the next cycle-deal to be applied, and let $R_{\mu}$ and $N_{\mu}$ be respectively the set of resources and the set of agents involved in $\mu$. We must prove that the allocation $\pi[\mu]$ can be reached by a sequence of swap-deals.	When $|R_{\mu}| = 2$, $\mu$ is actually a swap-deal and the claim is trivially true. 
	
	Assume then that $|R_{\mu}|>2$, let us consider the instance $I'$ obtained from $I$ by restricting $\resourceSet$ to $R_\mu$ and $\agentSet$ to $N_\mu$. In this restricted instance, the allocation $\pi^*$ where each agent receives her top resource is feasible. Indeed, it is the allocation obtained by applying $\mu$ (by definition of $\mu$ in TTC). This allocation is trivially Pareto-optimal as every agent has her top resources and it thus Pareto-dominates every other allocation. It is then the only Pareto-optimal allocation in $I'$. Remember from Observation \ref{obs:restricted-sp} that $I'$ is single-peaked. Then, by virtue of Theorem \ref{theo:PO}, this implies that $\pi^*$ must be reachable by improving swap-deals in $I'$. Overall, $\pi_{\mu}$ is then reachable by improving swap-deals in $I$. 
	
	The same argument can be stated for any step of the algorithm. By concatenating the sequences of improving swap-deals for every step of the algorithm, we then obtain a sequence of improving swap-deals leading to $\pi^{TTC}$ from $\pi^0$. \qed
\end{proof}

\begin{proposition}
	\label{prop:crawler-is-swap-reachable}
	Let $I = \tuple{\agentSet, \resourceSet, L, \pi^0}$ be an instance and let $\pi^C$ be the allocation returned by  Crawler. Then $\pi^C$ is reachable by swap-deals from $\pi^0$. 
\end{proposition}

\begin{proof}
	We show that every cycle-deal applied by  Crawler can be implemented as a sequence of improving swap-deals. For clarity reasons and without loss of generality, we assume that each agent $a_j$ currently holds resource $r_j$. 
	
	For a given iteration of the algorithm, consider agent $a_i$ who picks resource $r_k$ currently held by agent $a_k$. From the definition of the procedure, $a_k$ is on the left of $a_i$ (with respect to the single-peaked axis) and $a_k$ has already been considered at this step before considering $a_i$. In fact, all the agents between $a_k$ (included) and $a_i$ (excluded) on the left of the single-peaked axis have already been considered at the current iteration before reaching $a_i$. Moreover, all these agents have passed their turn because their peak is on their right. In other words, each agent $a_j$  between $a_k$ (included) and $a_i$ (excluded)  prefers the resource held by the agent $a_{j+1}$ on her right, that is:
	$$r_{j+1} \succ_j r_j, \forall j \in \{k, \ldots i - 1 \}.$$ 
	
	Let $\mu = \langle a_i, a_{i-1}, \cdots, a_{k+1}, a_k \rangle$ be the cycle-deal implemented by  Crawler at the current step. In this deal, $a_k$  gives her resource $r_k$ to $a_i$ and all the other agents of the deal give their resource to the next agent in the sequence which is the agent on their left with respect to the single-peaked axis. The decomposition of $\mu$ into a sequence of swap-deals consists in using agent $a_i$ as a hub for the swap-deals. Agent $a_i$ first exchanges with $a_{i-1}$ then, $a_{i}$ exchanges with $a_{i-2}$ and so on until $a_{i}$ performs a swap-deal with $a_{k}$. At the end, $a_i$ holds $r_k$ and each other agent $a_j$ in $\mu_i$ holds the resource initially held by $a_{j+1}$. The sequence of swap-deals is thus equivalent to $\mu_i$. 
	
	We now show that all these swap-deals are improving. In the first deal $\langle a_i, a_{i-1} \rangle$, $a_{i-1}$ receives the resource $r_i$ held by $a_i$ that she prefers to her current resource $r_{i-1}$ (as shown previously $r_{i} \succ_{i-1} r_{i-1}$). Simultaneously, agent $a_{i}$ receives the resource $r_{i-1}$ held by $a_{i-1}$ that she prefers to her current resource since her peak is on the left of $a_{i-1}$ (it is held by $a_k$), that is:
	$$r_{j} \succ_i r_{j+1}, \forall j \in  \{k, \ldots, i-1  \}.$$ The first deal is thus mutually improving.
	Concerning the next  swap-deal $\langle a_i, a_{j}\rangle$ with $j \in \{k, \ldots, i-2  \}$, $a_i$ exchanges  $r_{j+1}$ that she obtained from her previous swap-deal, against $r_j$ held by $a_{j}$. Agent $a_{j}$ receives the resource $r_{j+1}$ that she prefers to her current resource $r_{j}$ (as shown previously $r_{j+1} \succ_j r_j$). $a_{i}$ receives the resource $r_{j}$ that she prefers to $r_{j+1}$ (since $\forall j \in  \{k, \ldots, i-1\}, r_{j} \succ_i r_{j+1}$). All the swap-deals are thus mutually improving. $\qed$
\end{proof}

\medskip

These two results show that both the outcome of TTC and Crawler could be implemented as sequences of mutually beneficial swap-deals in single-peaked domains. 
This attenuates a bit a critical feature of these procedures. However, even though these long trading cycles can be broken down into swap-deals, they would still need to be pre-computed, and carefully coordinated. 

The proof of Proposition \ref{prop:crawler-is-swap-reachable} shows how to compute the set of swap-deals that achieves  Crawler outcome: for each cycle-deals $\mu = \langle a_i, a_{i-1}, \cdots, a_{k+1}, a_k \rangle$ implemented, $a_i$ is used as a hub to decompose the deal into mutually beneficial swap-deals. 
On the contrary, the proof of Proposition \ref{prop:ttc-is-swap-reachable} does not provide explicitly the corresponding sequence of deals. 
We show now how to compute the set of swap-deals that achieves the TTC outcome using Crawler. Note that for each cycle-deal selected by TTC, we can build a sub-instance of the initial instance and  apply Crawler on this sub-instance to  obtain the set of swap-deals that implements the cycle-deal. 
Since Crawler is Pareto-optimal, the following observation is straightforward: 

\begin{observation} \label{obs:outcomeCrawler}
	In a single-peaked housing market, if the allocation where each agent has her top resource is feasible (i.e. each agent has a different top resource), then the allocation is returned by  Crawler.
\end{observation}

We can now compute the set of swap-deals that achieves the TTC outcome using  Crawler.
For each deal $\mu$ implemented by TTC:
\begin{itemize}
	\item  If  $|\mu| \leq 2$, it is already a swap-deal and we are done. 
	\item If $|\mu|  > 2$, the deal has to be decomposed into a set of swap-deals. Let $I'$ be the sub-instance  restricted to the agents and the resources involved in $\mu$. This instance is guaranteed to remain  single-peaked (see Observation \ref{obs:subSet}). By definition of TTC, in $I'$, each agent obtains her top resource.  By Observation \ref{obs:outcomeCrawler}, the  Crawler outcome will be the same as the outcome of $\mu$ when applying Crawler to $I'$. The swap-deal decomposition of $\mu$ is then obtained by applying  Crawler on $I'$ and decomposing each cycle-deal $\mu'$ of  Crawler as explained above.  
\end{itemize}

Again, it should be kept in mind that in practice, 
a coordination mechanism would have to ensure that the agents indeed  execute the correct sequence of swap-deals (and thus the desired allocation would be obtained).

\medskip

Figure \ref{fig:summary-swaps} concludes this section by summarizing the different findings related the properties of swap dynamics. 

\begin{figure}[t]
	\centering
	\begin{tikzpicture}[scale = 0.8]
	\draw [rotate=30, fill={rgb:red,1;green,1;blue,1;white,50}] (0, 0) ellipse (4cm and 1.8cm);
	\node [align = center, font=\scriptsize\linespread{0.8}\selectfont] at (-1.6, -1.5) {\textbf{Allocations}\\\textbf{Pareto-dominating $\pi^0$}};
	
	\draw [rotate=30, fill={rgb:red,1;green,5;blue,8;white,80}] (1.2, 0) ellipse (2.3cm and 1.2cm);
	\node [align = center, font=\scriptsize\linespread{0.8}\selectfont] at (0.5, 0.25)  {\textbf{Swap-stable} \\ \textbf{allocations} \\ \textbf{reachable from $\pi^0$}};
	
	\node [align = center, font=\scriptsize\linespread{0.8}\selectfont] at (2.2, 0.7) {TTC \\ \textbullet};
	\node [align = center, font=\scriptsize\linespread{0.8}\selectfont] at (1.5, 1.5) {Crawler \\ \textbullet};

	\end{tikzpicture}  
	\caption{Summary of the results of Theorem \ref{theo:PO} and Propositions \ref{prop:notallPOreachable}, \ref{prop:ttc-is-swap-reachable} and \ref{prop:crawler-is-swap-reachable} for a given initial allocation $\pi^0$. Remember from Example \ref{ex:firstEx} that  TTC  and  Crawler can return different allocations.}
	\label{fig:summary-swaps}
\end{figure}
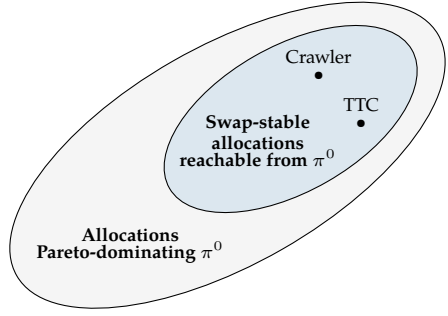

\subsection{Maximality of the single-peaked domain}

We now show that the single-peaked domain is maximal for the Pareto-optimality of swap dynamics: For every preference domain $D$ such that $\mathcal{SP}_\lhd$ is a proper subset of $D$ (for a given linear order $\lhd$) there exists an instance such that no swap dynamics can reach a Pareto-optimal allocation.

Before going through the proof, let us first start with a simple example. Consider the following profile and let $\lhd$ be the order $r_1 \lhd r_2 \lhd r_3$.
\begin{align*}
a_1: \enspace & r_1 \succ_{1} \heldc{r_2} \succ_{1} r_3 \\
a_2: \enspace & r_3 \succ_{2} \heldc{r_1} \succ_{2} r_2 \\
a_3: \enspace & r_2 \succ_{3} \heldc{r_3} \succ_{3} r_1
\end{align*}
It is clear that the profile is not single-peaked with respect to $\lhd$ (it actually is not for any order) as it is clearly not worst-restricted. Note that the preferences of agents $a_1$ and $a_3$ are single-peaked with respect to $\lhd$ but not those of agent $a_2$.

Take the underlined resources to form the initial allocation. It has been constructed as follows. Agents $a_1$ and $a_3$ receive their second best resource (respectively $r_2$ and $r_3$). For agent $a_2$, the violation of single-peakedness comes from her ranking of resources $r_1$ and $r_2$. She is thus allocated her most preferred resource among these two, i.e. $r_1$. It can easily be checked that this allocation is stable with respect to $C_2$ but is Pareto-dominated by $\tuple{r_1, r_3, r_2}$.

This construction is generalized to an arbitrary number of agents below.

\begin{theorem}
	\label{thm:MaxSPforSwap=>PO}
	Let $\resourceSet$ be a set of resources and $\lhd$ a linear order over $\resourceSet$. For every preference domain $D$ such that $\mathcal{SP}_\lhd \subsetneq D$, there exists an instance $I = \tuple{\agentSet, \resourceSet, L, \pi^0}$ defined over $D$ such that swap dynamics on $I$ do not reach a Pareto-optimal allocation. 
\end{theorem}

\begin{proof}
	Let us construct an instance $I = \tuple{\agentSet, \resourceSet, L, \pi^0}$ defined over $D$ such that swap dynamics on $I$ do not reach a Pareto-optimal allocation. 
	
	Without loss of generality and for the ease of the reader, let us assume that $\lhd$ is such that $r_1 \lhd r_2 \lhd \ldots \lhd r_n$.
	
	Since $\mathcal{SP}_\lhd \subsetneq D$, there exists a linear order $\succ^* \in D$ that is not single-peaked with respect to $\lhd$.  There exist then two resources $(r_{s - 1}, r_s) \in \resourceSet^2$ such that:
	$$\left.\begin{array}{r}
	r_{s} \lhd r_{s - 1} \lhd top(\succ^*), \\ 
	\text{or, } \enspace top(\succ^*) \lhd r_{s - 1} \lhd r_s
	\end{array}\right\} \text{ and } r_s \succ^*_i r_{s - 1}.$$
	Without loss of generality, let us assume that $top(\succ^*) \lhd r_{s - 1} \lhd r_s$. We will use index $t$ to refer to the top ranked resource in $\succ^*$, i.e., $r_t = top(\succ^*)$.
	
	Moreover, as $\mathcal{SP}_\lhd \subset D$, for every resource $r_i \in \resourceSet$ there exist in $D$ two linear orders $\succ_i^1$ and $\succ_i^2$ that are single-peaked with respect to $\lhd$ such that:
	$$top(\succ_i^1) = top(\succ_i^2) = r_i \text{ and } snd(\succ_i^2) = r_{i - 1}.$$
	Thus, in the preference order $\succ_i^1$ the top resource is $r_i$ and the other resources are not constrained. For $\succ_i^2$, the top resource is $r_i$ and the second best must be $r_{i - 1}$.
	
	We introduce the preference profile $L = \{\succ_{_i} \mid a_i \in \agentSet\}$ defined as follows:
	\begin{align*}
	&\succ_{_1} = \succ^*, \\
	&\succ_{_i} = \succ_{i - 1}^1, \quad \forall a_i \in \agentSet, i \in \{2, \ldots, t \}, \\
	&\succ_{_i} = \succ_{i}^2, \quad \forall a_i \in \agentSet, i \in \{t + 1, \ldots, s \}, \\
	&\succ_{_i} = \succ_{i}^1, \quad \forall a_i \in \agentSet, i \in \{s + 1, \ldots, n \}. 
	\end{align*}
	
	The initial allocation $\pi^0$ is then defined as:
	\begin{align*}
	&\pi^0_{a_1} = r_s, \\
	&\pi^0_{a_i} = top(a_i) = r_{i - 1}, \quad \forall a_i \in \agentSet, i \in \{2, \ldots, t \}, \\
	&\pi^0_{a_i} = snd(a_i) = r_{i - 1}, \quad \forall a_i \in \agentSet, i \in \{t + 1, \ldots, s \}, \\
	&\pi^0_{a_i} = top(a_i) = r_i, \quad \forall a_i \in \agentSet, i \in \{s + 1, \ldots, n \}.
	\end{align*}
	
	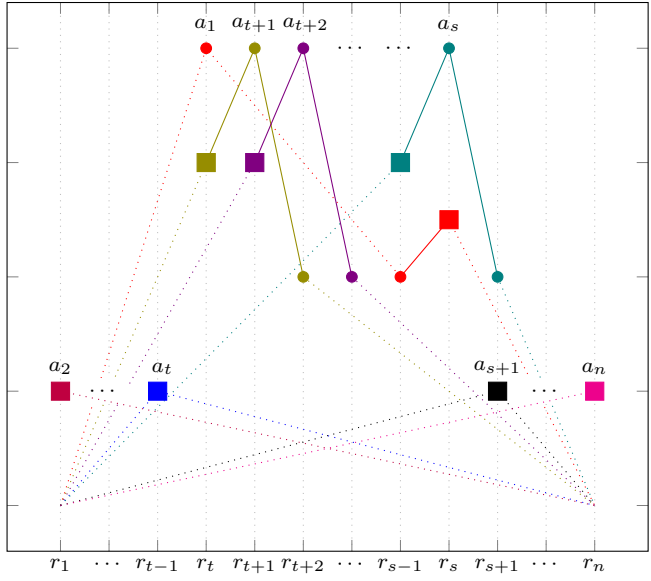
\begin{figure}
		\centering
		\begin{tikzpicture}
		\begin{axis}[
		scale only axis,
		legend pos = south east,
		xtick = {1, 3, 5, 7, 9, 11, 13, 15, 17, 19, 21, 23},
		xticklabels = {$r_1$, $\cdots$, $r_{t - 1}$, $r_{t}$, $r_{t + 1}$, $r_{t + 2}$, $\cdots$, $r_{s - 1}$, $r_{s}$, $r_{s + 1}$, $\cdots$, $r_n$},
		yticklabels = {,,,,,,,},
		xmajorgrids,
		major grid style={dotted}]
		
		\addplot[mark = square*, blue, mark options = {scale = 1.7}] coordinates {(5, 2)};
		\node[] at (axis cs:5.2, 2.2) {$a_t$};
		\addplot[dotted, blue, forget plot] coordinates {(1, 1) (5, 2) (23, 1)};
		
		\addplot[mark = *, red, forget plot] coordinates {(7, 5)};
		\addplot[dotted, red, forget plot] coordinates {(1, 1) (7, 5) (15, 3)};
		\addplot[mark = *, red, forget plot] coordinates {(15, 3) (17, 3.5)};
		\addplot[dotted, red, forget plot] coordinates {(17, 3.5) (23, 1)};
		\addplot[mark = square*, red, mark options = {scale = 1.7}] coordinates {(17, 3.5)};
		\node[] at (axis cs:7, 5.2) {$a_1$};
		
		\addplot[dotted, olive, forget plot] coordinates {(1, 1) (7, 4)};
		\addplot[mark = *, olive, forget plot] coordinates {(7, 4) (9, 5) (11, 3)};
		\addplot[dotted, olive, forget plot] coordinates {(11, 3) (23, 1)};
		\addplot[mark = square*, olive, mark options = {scale = 1.7}] coordinates {(7, 4)};
		\node[] at (axis cs:9, 5.2) {$a_{t + 1}$};
		
		\addplot[dotted, violet, forget plot] coordinates {(1, 1) (9, 4)};
		\addplot[mark = *, violet, forget plot] coordinates {(9, 4) (11, 5) (13, 3)};
		\addplot[dotted, violet, forget plot] coordinates {(13, 3) (23, 1)};
		\addplot[mark = square*, violet, mark options = {scale = 1.7}] coordinates {(9, 4)};
		\node[] at (axis cs:11.1, 5.2) {$a_{t + 2}$};
		
		\node[] at (axis cs:13, 5) {$\cdots$};
		\node[] at (axis cs:15, 5) {$\cdots$};
		
		\addplot[dotted, teal, forget plot] coordinates {(1, 1) (15, 4)};
		\addplot[mark = *, teal, forget plot] coordinates {(15, 4) (17, 5) (19, 3)};
		\addplot[dotted, teal, forget plot] coordinates {(19, 3) (23, 1)};
		\addplot[mark = square*, teal, mark options = {scale = 1.7}] coordinates {(15, 4)};
		\node[] at (axis cs:17, 5.2) {$a_{s}$};
		
		\addplot[dotted, forget plot] coordinates {(1, 1) (19, 2) (23, 1)};
		\addplot[mark = square*, mark options = {scale = 1.7}] coordinates {(19, 2)};
		\node[] at (axis cs:19, 2.2) {$a_{s + 1}$};

		\node[] at (axis cs:2.8, 2) {\colorbox{white}{$\cdots$}};
		\addplot[mark = square*, purple, mark options = {scale = 1.7}] coordinates {(1, 2)};
		\node[] at (axis cs:1,2.2) {$a_2$};
		\addplot[dotted, mark=none, purple, forget plot] coordinates {(1, 2) (23, 1)};

		\node[] at (axis cs:21, 2) {\colorbox{white}{$\cdots$}};
		
		\addplot[mark = square*, magenta, mark options = {scale = 1.7}] coordinates {(23, 2)};
		\node[] at (axis cs:23, 2.2) {$a_{n}$};
		\addplot[dotted, mark = none, magenta, forget plot] coordinates {(1, 1) (23, 2)};
		
		\end{axis}
		\end{tikzpicture}
		\caption{The instance constructed in the proof of Theorem \ref{thm:MaxSPforSwap=>PO}. The dotted lines represent the trend of the utilities, dots fixed points and squares the resource allocated to each agent.}
		\label{fig:MaxSPforPO}
	\end{figure}
	
	To get a better understanding of the instance $I = \tuple{\agentSet, \resourceSet, L, \pi^0}$ constructed in this proof, Figure \ref{fig:MaxSPforPO} presents the preference profile $L$ and the initial allocation $\pi^0$.
	
	We claim that $\pi^0$ is stable with respect to $C_2$ but not Pareto-optimal. Allocation $\pi^0$ is clearly not Pareto-optimal as the allocation in which every agent receives her top resource is feasible--- because no two agents have the same top resource---and this allocation clearly Pareto-dominates $\pi^0$. Let us now show that $C_2(\pi^0) = \emptyset$.
	
	First observe that every agent $a_i$, $i \in \{2, \ldots, t \} \cup \{
	s + 1, \ldots, n \}$ owns her top resource, hence can not be involved in an improving swap-deal. 
	
	Consider now agent $a_i, i \in \{t + 1, \ldots, s - 1 \}$. She owns her second most preferred resource which she would only trade against her top resource that is owned by agent $a_{i + 1}$. However, agent $a_{i + 1}$ is not interested in $\pi_{a_i}$, hence no improving swap-deal is possible.
	
	Finally, let us consider agent $a_s$ who owns her second most preferred resource $r_{s - 1}$ and whose top resource $r_s$ is owned by agent $a_1$. By the hypothesis that $a_1$'s preferences are not single-peaked, we have $r_{s} \succ_1 r_{s - 1}$. Once again, there is no improving swap-deal involving agent $a_s$.
	
	Overall, we have proved that $C_2(\pi^0) = \emptyset$, hence any swap dynamic returns $\pi^0$ on $I$. As $\pi^0$ is Pareto-dominated, this concludes the proof. $\qed$
\end{proof}

This result is particularly interesting since it shows that the single-peaked domain captures in a ``tight'' way the domain under which swap dynamics are Pareto-optimal (in the vein of similar results obtained by \citet{chevaleyre2010simple} in different settings). 

Note that this is not a characterization result: there are some domains that are not single-peaked but for which swap dynamics return Pareto-optimal allocations.

\begin{example}
	Let us consider the following preference profile:
	\begin{align*}
	a_1: \enspace & r_1 \succ_{_1} r_2 \succ_{_1} r_3 \\
	a_2: \enspace & r_1 \succ_{_2} r_3 \succ_{_2} r_2 \\
	a_3: \enspace & r_3 \succ_{_3} r_2 \succ_{_3} r_1
	\end{align*}
	This profile is not single-peaked over any linear order $\lhd$: the triplet $\tuple{r_1, r_2, r_3}$ is a witness of the violation of the worst-restrictedness condition (Proposition \ref{prop:SP=>WR}), however we can show that for every initial allocation, swap dynamics return a Pareto-optimal allocation. Let us consider the different initial allocations:
	\begin{enumerate}
		\item $\pi^0 = \tuple{r_1, r_2, r_3}$: $C_2(\pi^0) = \emptyset$ and $\pi^0$ is Pareto-optimal.
		\item $\pi^0 = \tuple{r_1, r_3, r_2}$: $C_2(\pi^0) = \emptyset$ and $\pi^0$ is Pareto-optimal.
		\item $\pi^0 = \tuple{r_2, r_1, r_3}$: $C_2(\pi^0) = \emptyset$ and $\pi^0$ is Pareto-optimal.
		\item $\pi^0 = \tuple{r_2, r_3, r_1}$: two swap-deals are possible:
		\begin{enumerate}
			\item $\tuple{a_1, a_3}$ which leads to case 2.
			\item $\tuple{a_2, a_3}$ which leads to case 3.
		\end{enumerate}
		\item $\pi^0 = \tuple{r_3, r_1, r_2}$: one swap-deal is possible: $\tuple{a_1, a_3}$ which leads to case 3.
		\item $\pi^0 = \tuple{r_3, r_2, r_1}$: three swap-deals are possible: 
		\begin{enumerate}
			\item $\tuple{a_1, a_3}$ which leads to case 1.
			\item $\tuple{a_1, a_2}$ which leads to case 4.
			\item $\tuple{a_2, a_3}$ which leads to case 5.
		\end{enumerate}
	\end{enumerate}
\end{example}

This example does not contradict Theorem \ref{thm:MaxSPforSwap=>PO}: Swap dynamics are  Pareto-optimal on the domain $D = \{\succ_{_1}, \succ_{_2}, \succ_{_3}\}$ because $D$ does not include every single-peaked linear order over $\lhd$. For instance take $\lhd$ to be $r_1 \lhd r_2 \lhd r_3$, extend $D$ by adding the two following linear orders: $r_2 \succ r_1 \succ r_3$ and $r_2 \succ r_3 \succ r_1$ so that $\mathcal{SP}_\lhd \subset D'$ and Theorem \ref{thm:MaxSPforSwap=>PO} will apply.

\section{The ``price of'' swap dynamics}
\label{sec:pricesOf}

It is natural to now ask to what extent swap dynamics induce a cost in terms of social welfare. 
In that perspective we  will first discuss the ``standard''  \emph{price of anarchy} ($\mathit{PoA}$)  \citep{KoutsoupiasEtAl2009,anshelevich2013anarchy}, that is, the (worst-case, over all instances) ratio between the worst stable outcome and the social welfare optimum. More formally, for an allocation procedure $M$:

$$
\mathit{PoA}_v(M) = \max_{I \in \mathcal{I_D}} \frac{\max v(\pi)}{\min_{\pi \in M(I)} v(\pi)} \text{~,~with~} v \in \{ark, mrk\}.
$$


But, as we have discussed, swap dynamics constitute a family of procedures. 
In that context, a perhaps more relevant metric is the ratio between the best and worst stable outcome which can be obtained by such procedures. In particular, this would tell us the price to pay for not being guided by a central planner (in the selection of deals) towards maximizing our social welfare notion. We call this notion the \emph{price of dynamics} ($\mathit{PoD}$), and for a family of dynamics $\mathcal{M}$, we define it as follows:

$$
\mathit{PoD}_v(\mathcal{M}) = \max_{\substack{I \in \mathcal{I_D}\\ M,M' \in \mathcal{M}}} \frac{\max_{\pi \in M'(I)} v(\pi)}{\min_{\pi \in M(I)} v(\pi)} \text{~,~with~} v \in \{ark, mrk\}.
$$
Note that this notion is not specific to swap dynamics and can be applied to any family of dynamics. Also, observe that an upper bound on the $\mathit{PoA}$ also applies to the $\mathit{PoD}$. 

These definitions are parametrized by the social welfare notion considered. 
In terms of average rank, we recall that \citet{damamme2015power} established that the price of anarchy is 2 for swap dynamics \emph{in the general domain}. 
While the upper bound remains valid in our restricted domain, the instance exhibited to show this bound to be tight in \citep{damamme2015power} violates single-peakedness. We now show that this result still holds under our domain restriction.

\begin{proposition}
	\label{prop:poa-is-2-in-SP}
	For any $M \in \mathcal{M}_2$, $\mathit{PoA}_{ark}(M)$ is $2$ in the single-peaked domain.
\end{proposition}

\begin{proof}
	
	Let us consider the following single-peaked instance involving $n$ agents, and assume without loss of generality $n$ to be odd. Take for now the shaded allocation as the initial allocation.\footnote{The underlined allocation will only be relevant in forthcoming proofs.}
	
	\begin{center}
		\setlength{\tabcolsep}{2pt}
		\begin{tabular}{rccccccccccccc}
			$a_1$ : & \fcolorbox{black}{lightgray}{$r_{n - 1}$} & $\succ$ & $r_n$ & $\succ$ & $r_{n - 2}$ & $\succ$ & $r_{n - 3}$ & $\cdots$ & $r_3$ & $\succ$ & \underline{$r_2$} & $\succ$ & $r_1$ \\ 
			$a_2$ : & \held{r_{1}} & $\succ$ & \heldb{r_2} & $\succ$ & \underline{$r_{3}$} & $\succ$ & $r_{4}$ & $\cdots$ & $r_{n-2}$ & $\succ$  & $r_{n - 1}$ & $\succ$ & $r_n$ \\ 
			$a_3$ : & \held{r_{2}} & $\succ$ & $r_1$ & $\succ$ & \heldb{r_{3}} & $\succ$ & \underline{$r_{4}$} & $\cdots$ &  $r_{n-2}$ & $\succ$ & $r_{n - 1}$ & $\succ$ & $r_n$ \\ 
			\multicolumn{1}{c}{$\vdots$} & \multicolumn{1}{c}{$\vdots$} & \multicolumn{1}{c}{$\vdots$} & \multicolumn{1}{c}{$\vdots$} & \multicolumn{1}{c}{$\vdots$} & \multicolumn{1}{c}{$\vdots$} & \multicolumn{1}{c}{$\vdots$} & \multicolumn{1}{c}{$\vdots$} & \multicolumn{1}{c}{$\vdots$} & \multicolumn{1}{c}{$\vdots$} & \multicolumn{1}{c}{$\vdots$} & \multicolumn{1}{c}{$\vdots$} \\
			$a_{n - 2}$ : & \held{r_{n - 3}} & $\succ$ & $r_{n - 4}$ & $\succ$ & $r_{n -5}$ & $\succ$ & $r_{n-4}$ &  $\cdots$ & \heldb{r_{n - 2}} & $\succ$ & \underline{$r_{n - 1}$} & $\succ$ & $r_n$ \\ 
			$a_{n - 1}$ : & \held{r_{n - 2}} & $\succ$ & $r_{n - 3}$ & $\succ$ & $r_{n - 4}$ & $\succ$ & $r_{n - 5}$ & $\cdots$ &  $r_{n-1}$ & $\succ$ & \heldb{r_{n}} & $\succ$ & \underline{$r_1$} \\ 
			$a_n$ : & \heldb{r_{1}} & $\succ$ & $r_{2}$ & $\succ$ & $r_{3}$ & $\succ$ & $r_{4}$ & $\cdots$ &  $r_{n-2}$ & $\succ$ & $r_{n-1}$ & $\succ$ & \underline{\held{r_n}} \\ 
		\end{tabular}
	\end{center}
	

	Let $\pi$ be the shaded allocation, and $\pi^*$ be the squared allocation ($a_1$ holds $r_{n-1}$ in both cases).   
	Observe that both $\pi$ and $\pi^*$ are Pareto-optimal allocations. 
	For $\pi^*$ this is obvious as only $a_n$ does not hold her top resource.   
	For $\pi$, notice that $a_2$ would only wish to swap for $r_1$, which is held by an agent who ranks it first. Then $a_3$ would only wish to swap for $r_1$ or $r_2$, and so on until $a_{n-2}$. 
	Finally, $a_{n-1}$ may swap with anyone but $a_n$, but no one wants to swap with her.  
	
	We see that $ark(\pi^*) = [(n-1) \cdot n + 1]/n$, while $ark(\pi) = [n\cdot (n+1)/2 + n -1]/n$, thus this instance shows that (asymptotically) the $\mathit{PoA}$ is at least 2. In \citet{damamme2015power} it is shown that 2 is an upper bound for the $\mathit{PoA}$ of any swap-deal procedure. We thus conclude that the $\mathit{PoA}$ is 2 here, as in the general domain.
	$\qed$
\end{proof}

Following \cite{damamme2015power}, we can also make some further observations. In fact, as both allocations are Pareto-optimal, this instance shows that the $\mathit{PoA}$ of any procedure satisfying individual rationality must be at least 2. Indeed, taking the shaded allocation $\pi$ as the initial allocation, any individually rational procedure would output $\pi$.  
Furthermore, as any Pareto-optimal allocation is stable for swap-deals, our upper bound remains valid.  

\begin{observation}
	\label{prop:PoA-TTC-is-2}
	For any individually rational allocation procedure $M$, 
		$\mathit{PoA}_{ark}(M)$ is $2$ in the single-peaked domain. This holds in particular for TTC and  Crawler.
\end{observation}

To address the price of dynamics, we must now make use of a different initial allocation, and show that both allocations are reachable by sequences of swaps. This would demonstrate that this gap could also only be due to the selection heuristics.  

\begin{proposition}
	\label{prop:PoD-ark-is-2-in-SP}
	For the family of swap dynamics $\swapfamily$, $\mathit{PoD}_{ark}(\swapfamily)$ is $2$ in the single-peaked domain.
\end{proposition}

\begin{proof}
	Consider again the instance presented in the proof of Proposition \ref{prop:PoD-ark-is-2-in-SP}.
	Take now the initial allocation to be $\pi^0= \langle r_2, r_3, r_4, \cdots ,  r_{n-1}, r_1, r_n\rangle$, the underlined allocation.  
	We first show that the shaded allocation $\pi$ can be reached from $\pi^0$ via improving swap-deals. First the deal $\tuple{a_n,a_{n-1}}$ is implemented. Then $a_1$ acts as a hub, and the following sequence is implemented $\tuple{a_1,a_2} \dots \tuple{a_1,a_{n-2}}$. 
	
	Now, the squared allocation $\pi^*$ can also be reached with swap-deals. Agent $a_1$ will act as a hub for odd agents, while $a_{n-1}$ will do the same for even agents. In other words, the following sequence is implemented: $\tuple{a_1,a_3}, \tuple{a_1,a_5} \dots \tuple{a_1,a_{n-2}}$ (note that $a_n$ gets her final resource from the start), followed by the sequence of swaps $\tuple{a_{n-1},a_2}, \tuple{a_{n-1},a_4} \dots \tuple{a_{n-1},a_{n-3}}$.
	
	The result follows then from the computations of the average rank presented in Proposition \ref{prop:PoD-ark-is-2-in-SP}. \qed
	
\end{proof}

Finally, we turn our attention to the mininimum rank counterparts of the same notions. Note that $n$ is certainly an upper bound in that case, as this is highest possible ratio between two allocations.

\begin{proposition} \label{prop:minRankPoA}
	For any $M \in \mathcal{M}_2$, $\mathit{PoA}_{mrk}(M)$ is $\Theta(n)$  in the single-peaked domain.
\end{proposition}

\begin{proof}
	Let us consider the following single-peaked instance involving $n$ agents. Take the shaded allocation as the initial allocation.\footnote{Again, the underlined allocation will only be relevant in a later proof.} 
	Let $\pi$ be the shaded allocation, and $\pi^*$ be the squared allocation. 
	\begin{center}
		\setlength{\tabcolsep}{4pt}
		\begin{tabular}{rccccccccccc}
			$a_1$ : & \heldb{r_{1}} & $\succ$ & $r_2$ & $\succ$ & $r_{3}$ & $\succ$ & $r_{4}$ & $\cdots$ & \underline{$r_{n - 1}$} & $\succ$ & $r_n$ \\ 
			$a_2$ : & \heldb{r_{2}} & $\succ$ & \held{r_3} & $\succ$ & \underline{$r_1$} & $\succ$ & $r_{4}$ & $\cdots$ & $r_{n - 1}$ & $\succ$ & $r_n$ \\ 
			$a_3$ : & \heldb{r_{3}} & $\succ$ & \held{r_4} & $\succ$ & \underline{$r_2$} & $\succ$ & $r_{1}$ & $\cdots$ & $r_{n - 1}$ & $\succ$ & $r_n$ \\ 
			$a_4$ : & \heldb{r_{4}} & $\succ$ & \held{r_5} & $\succ$ & \underline{$r_3$} & $\succ$ & $r_{2}$ & $\cdots$ & $r_{n - 1}$ & $\succ$ & $r_n$ \\ 
			\multicolumn{1}{c}{$\vdots$} & \multicolumn{1}{c}{$\vdots$} & \multicolumn{1}{c}{$\vdots$} & \multicolumn{1}{c}{$\vdots$} & \multicolumn{1}{c}{$\vdots$} & \multicolumn{1}{c}{$\vdots$} & \multicolumn{1}{c}{$\vdots$} & \multicolumn{1}{c}{$\vdots$} & \multicolumn{1}{c}{$\vdots$} & \multicolumn{1}{c}{$\vdots$} & \multicolumn{1}{c}{$\vdots$} & \multicolumn{1}{c}{$\vdots$} \\
			$a_n$ : & \held{r_{2}} & $\succ$ & $r_1$ & $\succ$ & $r_{3}$ & $\succ$ & $r_{4}$ & $\cdots$ & $r_{n - 1}$ & $\succ$ & \heldb{\heldc{r_n}} 
		\end{tabular}
	\end{center}

	It is easy to see that $mrk(\pi^*) = n-1$, while $mrk(\pi) = 1$, thus $\mathit{PoA}$ is $\Omega(n)$.  Clearly, the $\mathit{PoA}$ cannot be worse, thus it is $\Theta(n)$. $\qed$
\end{proof}

\begin{proposition}
	\label{prop:pod-mrk-is-2-in-SP}
	$\mathit{PoD}_{mrk}(\mathcal{M}_2)$ is $\Theta(n)$ in the single-peaked domain.
\end{proposition}

\begin{proof} Let us consider the single-peaked instance of Proposition \ref{prop:minRankPoA} and let us assume that the initial allocation is now $\pi^0= \langle r_{n-1}, r_1, r_2, \cdots , r_{n-2}, r_n\rangle$, the underlined allocation. 
	
	From this initial allocation $\pi^0$, it is possible to reach the shaded allocation $\pi = \langle r_{1}, r_2, r_3, \cdots , r_{n-1}, r_n\rangle$ with improving swap deals. This allocation is reached by performing $n-2$ swap-deals $\tuple{a_1, a_i}$, $\forall i \in \{n-1, \ldots, 2\}$ starting with $\tuple{a_1,a_{n-1}}$ and finishing with $\tuple{a_1,a_{2}}$.
	
	From $\pi^0$, it is also possible to reach $\pi^* =\langle r_{1}, r_3, r_4, \cdots  , r_{n}, r_2\rangle$ (the squared allocation) which is both Pareto-optimal and optimal in terms of mininimum rank. If $n$ is even, we perform the following sequence of couples of swap-deals $\tuple{a_{i-1}, a_n}$ and $\tuple{a_{i-2}, a_1}$ $\forall i \in \{n, n-2, \ldots, 6, 4\}$. If $n$ is odd, we perform the same sequence of couples of swap-deals but $\forall i \in \{n, n-2, \ldots, 7, 5\}$ and add two final swap-deals $\tuple{a_2, a_n}$ and $\tuple{a_1, a_n}$. 
	
	The result follows from the previous computations of the minimum rank in these two allocations.\qed
	
\end{proof}

These different theoretical results remain a worst-case analysis. To complete the picture, it would be valuable to know how well these swap dynamics do in practice. This is what we investigate in the next section.

\section{Experimental study}
\label{sec:expe}

Our objective in this section is to get some insights about the empirical behaviour of swap dynamics. Our study includes swap dynamics based on the different history-based selection heuristics introduced in Section \ref{sec:decentralProcedures}, as well as two other selection heuristics added for comparison: one which takes as input more than the mere history of deals ($M_2$-PW), and another one which relaxes the constraint of using swaps only ($M_3$-U).  
For completeness, we also compare their performance to that of TTC and  Crawler  (keeping in mind though that these two centralized procedures are not specifically designed to optimize our social welfare measures).

We now give the full detail of our protocol by specifying (i) how deals are selected in the swap dynamics, (ii) how preferences are generated, as well as (iii) the full specification of the parameters used in the experiments.

\subsection{Experimental protocol}

\subsubsection{Selection heuristics for swap dynamics}

Concerning swap-deal procedures, our study includes the history-based selection heuristics already introduced in Section \ref{sec:decentralProcedures}: $M_2$-RRA, $M_2$-RRP, $M_2$-RRR, $M_2$-U and $M_2$-RM.    
We also include for comparison a preference-based heuristics ($M_2$-PW) and a procedure allowing cycles involving up to 3 agents ($M_3$). More precisely: 
\begin{itemize}
	\item \emph{Priority to the worst-off agent} ($M_2$-PW): agents are ordered considering the rank of the resource they own from the one with the lowest rank to the one with the highest rank.  Agents are then paired in a round-robin fashion like $M_2$-RRA does following this ordering. Note that $M_2$-PW is more demanding than history-based heuristics as it requires some information about the agent's current rank which can only be collected via a central entity.
	\item \emph{Uniform up to three agents} ($M_3$-U): a deal is selected uniformly at random among all  possible deals involving 2 or 3 agents. If the deal is rational for all the agents involved, it is implemented. 
\end{itemize}

\subsubsection{Generation of single-peaked preferences}
\label{sec:SP-cultures}
Different methods can be envisioned to generate single-peaked preferences. We consider impartial culture for single-peaked domain (IC-SP) and uniform peak for the single-peaked domain (UP-SP). 

Single-peaked preferences under impartial culture (IC-SP) are drawn using the method proposed by \cite{walsh2015generating}. Given an axis, single-peaked preferences are built recursively from the end (i.e. the worst resource of the agent) to the top resource. At each iteration, the next resource in the preference order is randomly selected between the two extremes of the axis. The selected resource is then removed from the axis and so on until the axis is empty. 

In the uniform peak culture (UP-SP), presented by \cite{conitzer2009eliciting}, preferences are constructed by first picking uniformly at random a resource to be the peak. The second-highest ranked resource is chosen with equal probability from the two adjacent alternatives, and so on until a full order is obtained. 

As already mentioned by \cite{walsh2015generating}, the probabilities of the preference orders significantly differ from one method to another.  Under IC-SP, each single-peaked preference order has a uniform probability $\nicefrac{1}{2^{n-1}}$ to be selected. On the contrary, under UP-SP, probabilities over preference orders are not uniform.  In fact, the peak is uniformly drawn (with probability $\nicefrac{1}{n}$) and single-peaked preferences are then built from this peak. Since there is only one preference order with its peak at one end of the axis, these orders are more likely to be drawn than preference orders with a peak in the middle of the axis, for instance. 

Let us consider the case where $n=5$ and the axis is $r_1\lhd r_ 2 \lhd  r_3 \lhd r_4 \lhd r_5$. Preference orders $r_1 \succ r_2 \succ r_3 \succ r_4 \succ r_5$  and   $r_5 \succ r_4 \succ r_3 \succ r_2 \succ r_1$ both have a probability $\nicefrac{1}{5}$ to be generated whereas the probability to generate some preference order with the peak $r_3$ is $\nicefrac{1}{5}$, and there are 6 orders with this peak. Under IC-SP, the probability to generate the preference order   $r_1 \succ r_2 \succ r_3 \succ r_4 \succ r_5$ is $\nicefrac{1}{16}$. 

Closer inspection reveals that the frequency of each rank for a given resource is more evenly distributed with UP-SP than with IC-SP. 
For instance, under IC-SP,  resources on the end-sides of the single-peaked axis have a high probability  to be ranked among the top  resources of an agent  but they have a very low probability to be ranked in the second half of the preference linear order. Under UP-SP,  these resources have a high probability to be selected as  the top resource of an agent but they also have a more uniform probability distribution among all other possible ranks. A resource on the end-side of the single-peaked axis is thus more likely to obtain a low ranking under UP-SP.  

This suggests that profiles under UP-SP are more diverse than under IC-SP. To quantify this more precisely, we have computed two diversity indicies suggested by \cite{HashemiEndrissECAI2014}: the sum of Kendall's tau and the sum of Spearman distances.\footnote{The Kendall's tau distance between two preference order $\succ_i$ and $\succ_{i'}$ over $\resourceSet$ is the number of pairs $(r, r') \in \resourceSet^2$ such that $\succ_i$ and $\succ_{i'}$ do not rank $r$ and $r'$ in the same order.
		The Spearman distance between two preference order $\succ_i$ and $\succ_{i'}$ over $\resourceSet$ is defined as $\sum_{r \in \resourceSet} |rank_{a_i}(r) - rank_{a_{i'}}(r)|$.
		In both cases, the diversity index of a profile $L$ is the sum of the distance between every pair of preference orders $(\succ_i, \succ_{i'})$ of $L$.}
	The results confirm that UP-SP generates more diverse profiles on average compared to IC-SP.
	Take for instance the case of seven resources ($n = 7)$. We generated 10000 preference profiles under both UP-SP and IC-SP. With respect to the Kendall's tau index, profiles under UP-SP achieved an average diversity score of 75\% (normalized over the empirically observed maximum) while this value reaches only 48\% in the case of IC-SP. With the Spearman distance, the numbers are still in favour of UC-SP: 85\% versus 64\%.

\subsubsection{Experiments conducted}

	We conducted two types of experiments: 
	\begin{itemize}
		\item[(i)] A study of the average efficiency and fairness of swap dynamics. 
		The \emph{efficiency} (respectively \emph{fairness}) ratio is defined as the ratio between the average  (respectively minimum) rank realized by the procedure and the optimal average (respectively min) rank achievable for the instance (disregarding the individual rationality constraint). The optimal values are obtained using matching techniques \citep{garfinkel1971improved}. 
		These investigations can be viewed as the empirical counterpart of the worst-case analysis of Section~\ref{sec:pricesOf}. 
		
		To isolate the influence of the individual rationality  (IR) constraints alone, we also include  
		the rules ``$\max ark$ IR'' and ``$\max mrk$ IR'' which respectively returns the allocation maximizing the average and the minimum ranks under the constraint of being individually rational. 
		We also study the performance of swap dynamics with respect to TTC and  Crawler. 
		\item[(ii)] A study of the number of swaps  performed by swap dynamics  with respect to the number of cycles performed by TTC and Crawler. Of course this comparison needs to acknowledge that the size of cycles may be very different in both approaches. We thus also investigate the size of the cycles computed by TTC and Crawler. 
	\end{itemize}

\subsubsection{Parameters of the experiments}

For every experiment, we considered several number of resources, varying from $n = 2$ to $n = 60$ and we randomly generated 1000 instances in each case. For a given instance, the initial allocation is selected uniformly at random among all possible allocations.
For every instance, we ran the different procedures and report the results for a fixed number of resources, averaged over the 1000 instances.

\subsection{Analysis of the results}

\subsubsection{Efficiency and fairness of swap dynamics}

	\paragraph{Comparison among swap dynamics.}
	Figure \ref{fig:expeSW_Swaps} presents the efficiency and fairness ratio for each swap dynamic for preferences generated both under the IC-SP model (left side) and under the UP-SP model (right side). Results for $\max mrk$ IR and $\max ark$ IR are also presented for comparison purposes.

	\begin{figure*}
		\includegraphics[width=\linewidth]{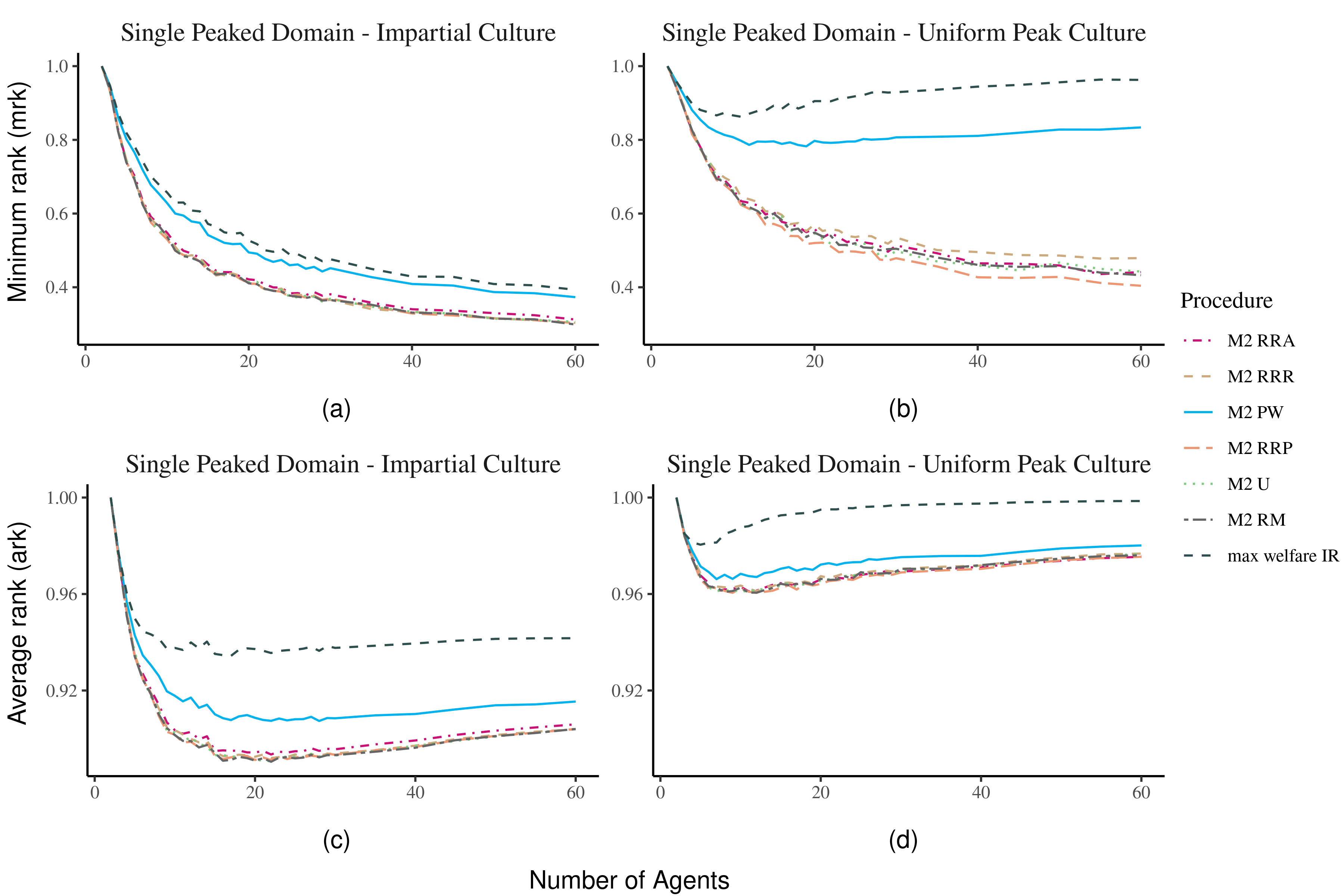}
		\caption{Average efficiency and fairness ratios for each heuristic of the swap-deal procedure and for each preference culture. For comparison baseline, the average maximum welfare achievable under individual rationality ($\max mrk$ IR for the upper part and $\max ark$ IR for the lower part) is also presented. The number of agents varies from 2 to 60.}
		\label{fig:expeSW_Swaps}
	\end{figure*}

	\revisedC{Regarding the minimum rank---Figures \ref{fig:expeSW_Swaps}(a) and \ref{fig:expeSW_Swaps}(b)---$M_2$-PW reaches} significantly better allocations than other dynamics. This is expected: this heuristic favours deals between agents holding low ranked resources and thus tends to improve the satisfaction of the poorest agents. 
	On the contrary, round-robin heuristics with their fixed ordering tend to always favour the same agents and thus often leads to lower minimum rank. However, notice that the randomized round-robin version ($M_2$-RRR) slightly corrects this.
	
	Interestingly, the performance of $M_2$-PW 
	is highly sensitive to the culture considered.
	This seems to be due to the way preferences are generated. 
	Indeed, the results for $\max mrk$ IR follow similar patterns, indicating that \revisedC{IC-SP (Figure \ref{fig:expeSW_Swaps}(a)) is much more constrained by the initial allocation than UP-SP (Figure \ref{fig:expeSW_Swaps}(b)) is.} Still, under UP-SP, $M_2$-PW returns close to optimal allocations, while other dynamics remain with a fairness ratio of about 50\%.
	Following this observation, as the allocation returned by $M_2$-PW can be taken as a witness of a good allocation reachable by swap dynamics, we see that the empirical minimum rank price of dynamics for our studied family of dynamics is much higher in UP-SP than in IC-SP. 
	
	\revisedC{Regarding the average rank of the outcomes---Figures \ref{fig:expeSW_Swaps}(c) and \ref{fig:expeSW_Swaps}(d)---all the heuristics} obtain very good results \revisedC{(above $90\%$  under IC-SP (Figure \ref{fig:expeSW_Swaps}(c)) and above $96\%$ under UP-SP (Figure \ref{fig:expeSW_Swaps}(d))).} It can be observed that all heuristics give very similar values.
	Recall that our price of anarchy results informed us that in principle, swap dynamics can return allocation with an efficiency ratio as low as $50\%$. Given these high performances, it is not surprising that very little room is left to observe difference among selection heuristics. Still, it is noteworthy that $M_2$-PW provides, again, the best results (while this selection heuristics doesn't seem designed to optimize this measure of welfare at first sight).  
	Note that our remark about the relative performance of P2-PW under the two cultures does not hold here: this suggests that only a few agents are left with low-ranked resources with our swap dynamics, even under UP-SP. 

\begin{figure*}
	\includegraphics[width=\linewidth]{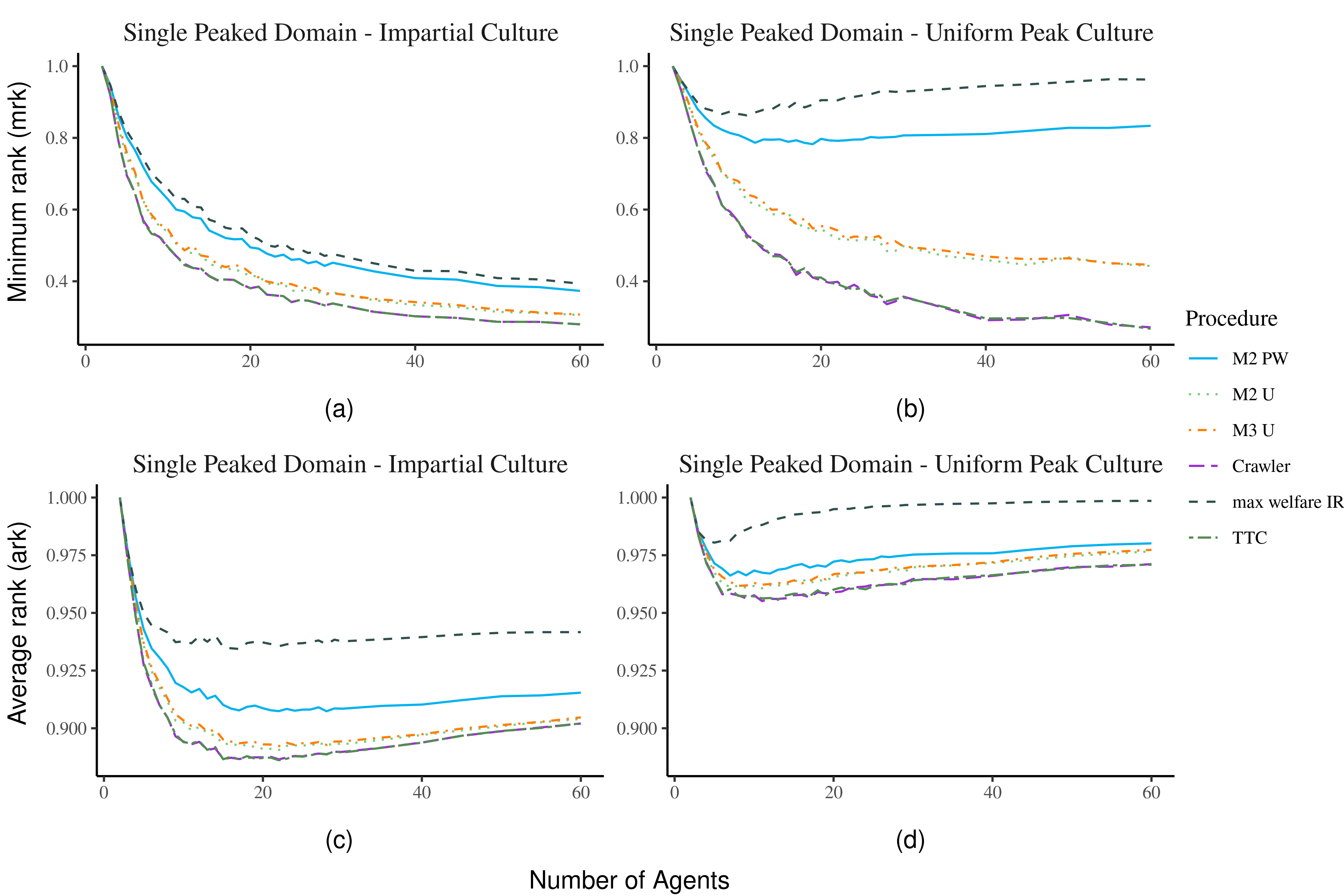}
	\caption{Average efficiency and fairness ratios for each procedure and for each preference domain. Curves for ``$\max$ welfare IR'' correspond to $\max mrk$ IR for the upper part of the figure and to $\max ark$ IR for the lower part. The number of agents varies from 2 to 60.}
	\label{fig:expeSW_All}
\end{figure*}

\paragraph{Comparison with TTC and Crawler.}
We now compare swap dynamics to other procedures.
	For the sake of readability, we only keep $M_2$-PW (which offered the best performance) and $M_2$-U (the most decentralized selection heuristics, requiring hardly any communication between the agents). Results are presented in Figure \ref{fig:expeSW_All}.

	\revisedC{Regarding the minimum rank---Figures \ref{fig:expeSW_All}(a) and \ref{fig:expeSW_All}(b)---a first observation} is that $M_2$-U performs significantly better than TTC and Crawler algorithm (which give similar results). The poor performances of these procedures, not designed for this purpose, are not surprising: they typically leave some agents with very little opportunities to be involved in a cycle-deal. For instance, by selecting top-trading cycles, TTC can remove  all the agents involved once a deal is implemented, thus significantly limiting the range of possible deals for the remaining agents. On the other hand, as discussed, $M_2$-PW favours low ranked agents and gives more opportunities to these agents to exchange their initial resources. 
	
	\revisedC{Regarding the average rank of the outcomes obtained by the different procedures---Figures \ref{fig:expeSW_All}(c) and \ref{fig:expeSW_All}(d)---}swap dynamics slightly outperform TTC and Crawler under both cultures.

	Finally, it can also be observed that $M_3$-U shows no significant difference with $M_2$-U under the same heuristics (uniform selection of the exchanges), either for the average or minimum rank: slightly increasing the size of the deals leads raises complex coordination issues with no evidence of improvements in terms of welfare. 

\medskip

	These experiments promote the relevance of swap dynamics: besides being simple to implement, they also provide very good results both in terms of average and minimum rank, close to the optimal when individual rationality is enforced. For the minimum rank and under UP-SP, we observed  that our  history-based swap dynamics may incur a significant cost, even though better allocations would be reachable by swaps (this was shown by comparison to a more specifically designed, not history-based, selection heuristic).

\subsubsection{Length of swap dynamics}

Figure \ref{fig:expeNbSwaps} represents the mean number of swap-deals performed (solid lines) when varying the size of the instances. For the sake of readability, we only show some selection heuristics (the other ones performed a number of swaps similar to either $M_2$-RRA or $M_2$-U). Dotted lines represent the highest and the lowest numbers of swaps registered for an instance of a given size (averaged over 1000 randomly generated instances).

It can be observed that the different selection heuristics are close in terms of the  average number of swaps. However, the number of swaps performed under UP-SP is significantly higher  than under IC-SP.  This phenomenon is related to  the method used to generate single-peaked preferences. 
According to the discussion of Section \ref{sec:SP-cultures}, UP-SP offers a greater diversity of profiles, which leads to more opportunities of exchanges. 
However, even though we noted in previous experiments that the results in terms of welfare were better in UP-SP than in IC-SP, note that the proportion is of a different order here: there are on average three times more swap-deals performed under UP-SP. 
	
	The higher diversity of profiles under UP-SP can also explain the higher variance in the number of swaps. Typically, when agents have completely opposite preference orders, the number of deals can greatly vary depending on which agents are encountered. This is much more likely to happen under UP-SP. 

\begin{figure*}
	\includegraphics[width=\linewidth]{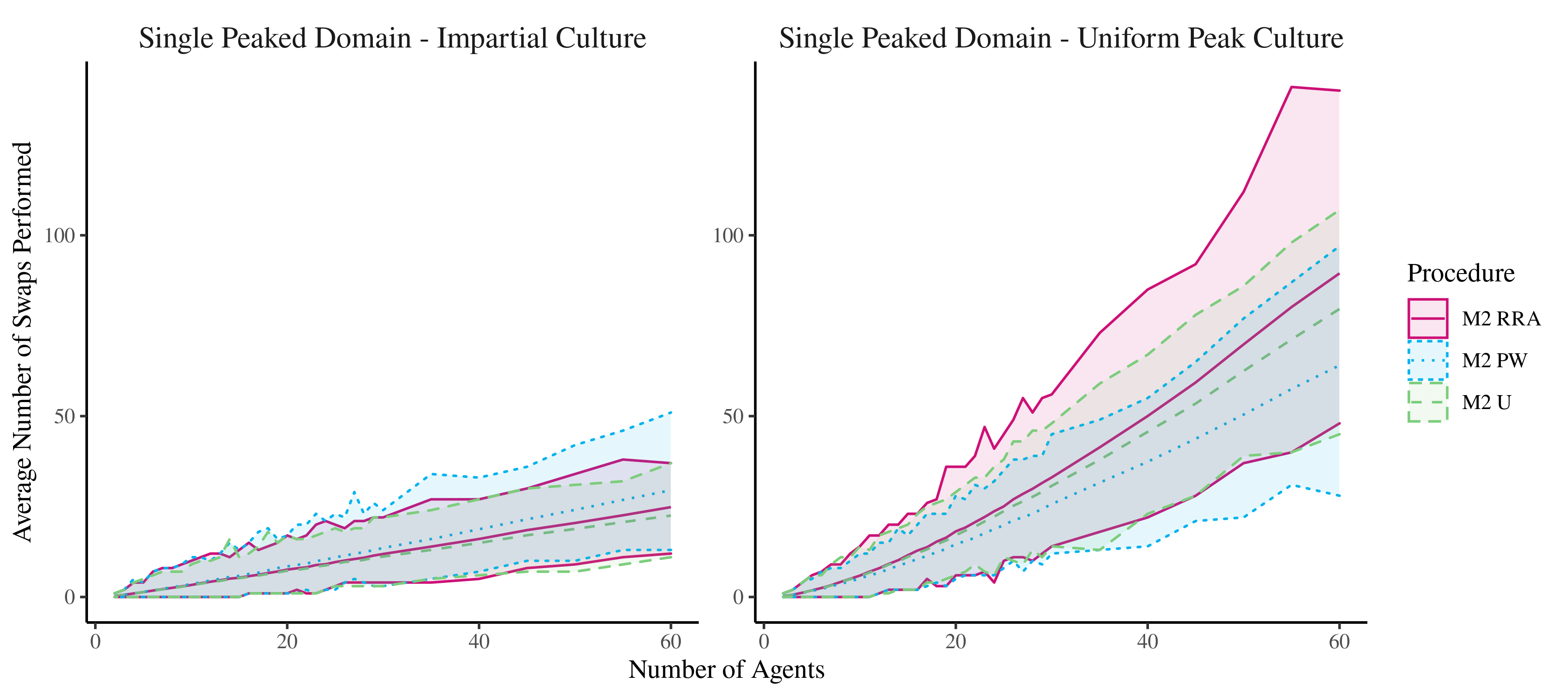}
	\caption{Average number of swap-deals performed, the filled area represents the range, from the minimum to the maximum. The number of agents varies from 2 to 60.}
	\label{fig:expeNbSwaps}
\end{figure*}
\begin{figure*}
	\includegraphics[width=\linewidth]{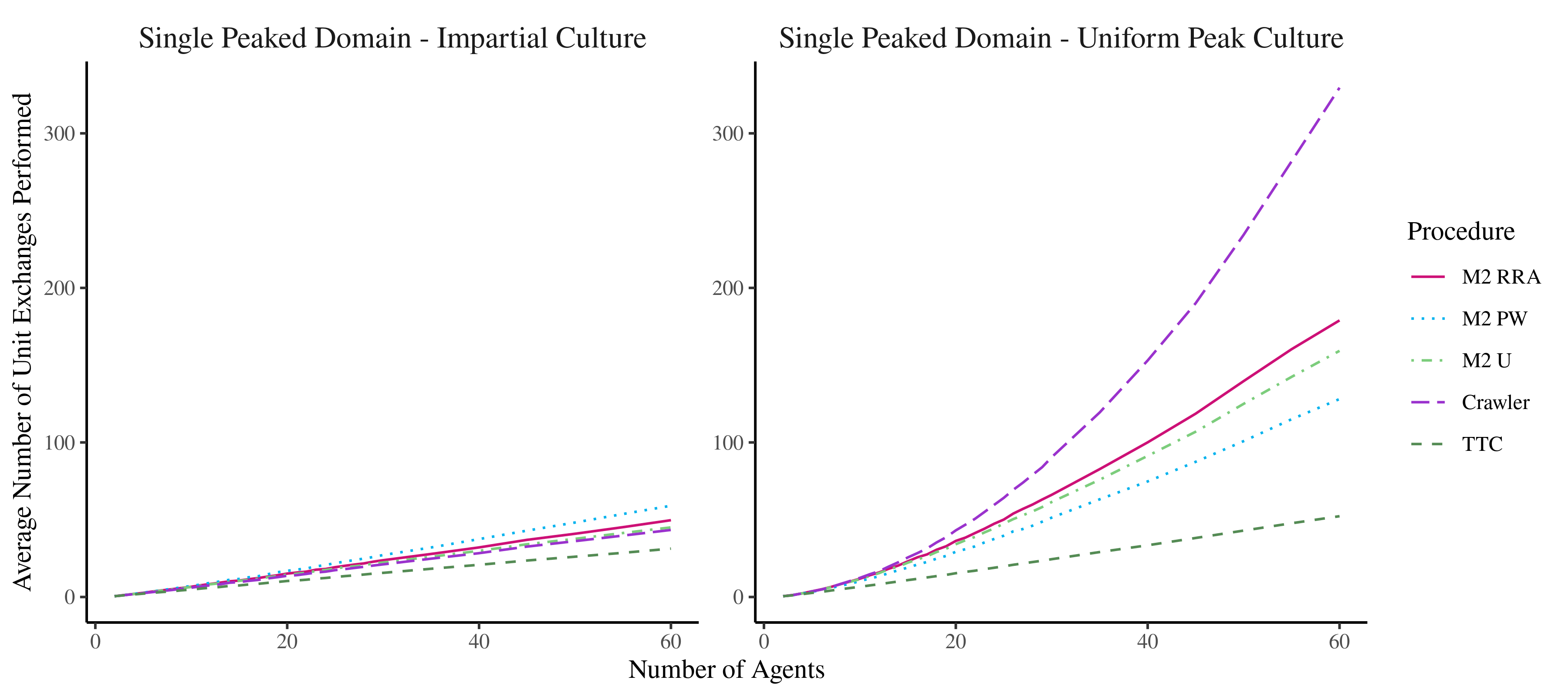}
	\caption{Average number of unit exchanges performed by some swamp dynamics, TTC and Crawler. A cycle-deal of length $k$ counts for $k$ unit exchanges. The number of agents varies from 2 to 60.}
	\label{fig:expeNbExchanges}
\end{figure*}
\begin{figure*}
	\includegraphics[width=\linewidth]{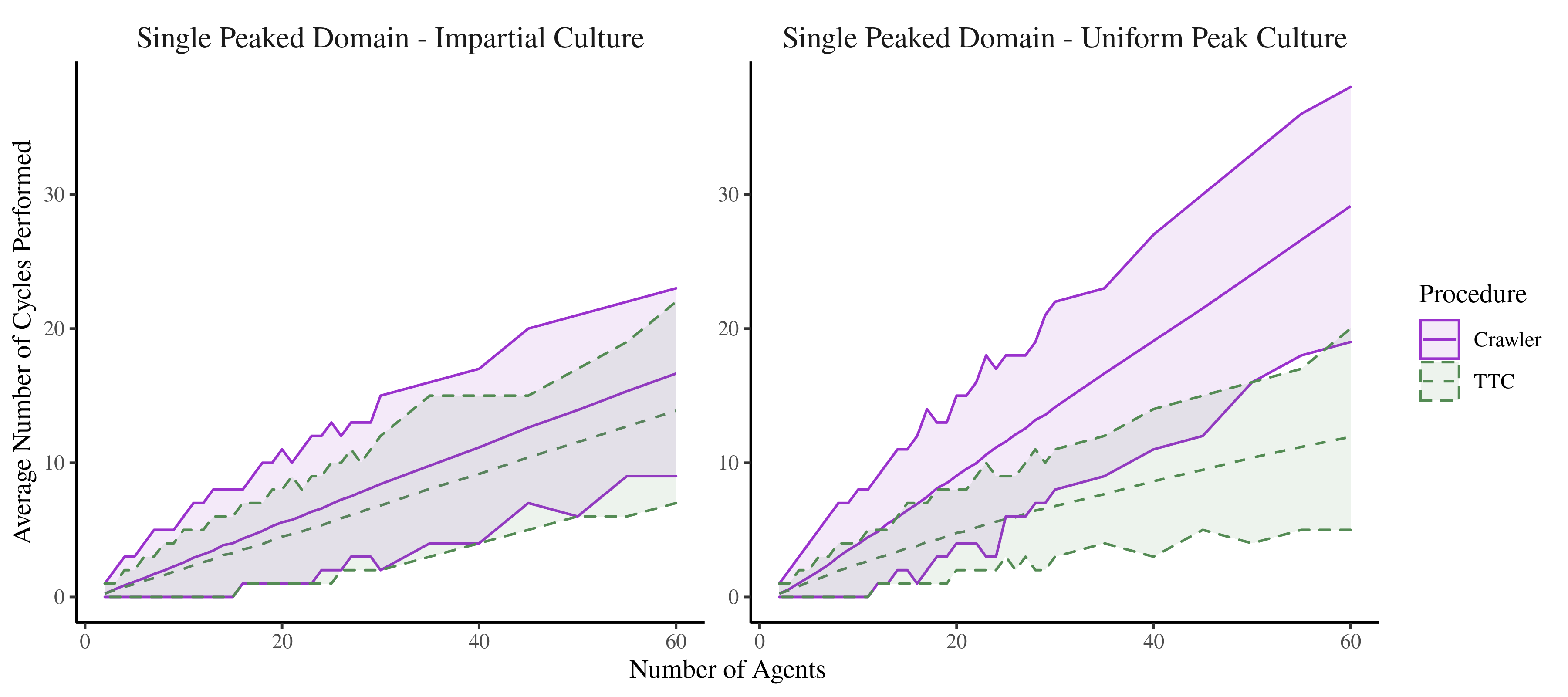}
	\caption{Average number of cycles performed by TTC and Crawler, the filled area represents the range, from the minimum to the maximum. The number of agents varies from 2 to 60.}
	\label{fig:expeNbCycle}
\end{figure*}
\begin{figure*}
	\includegraphics[width=\linewidth]{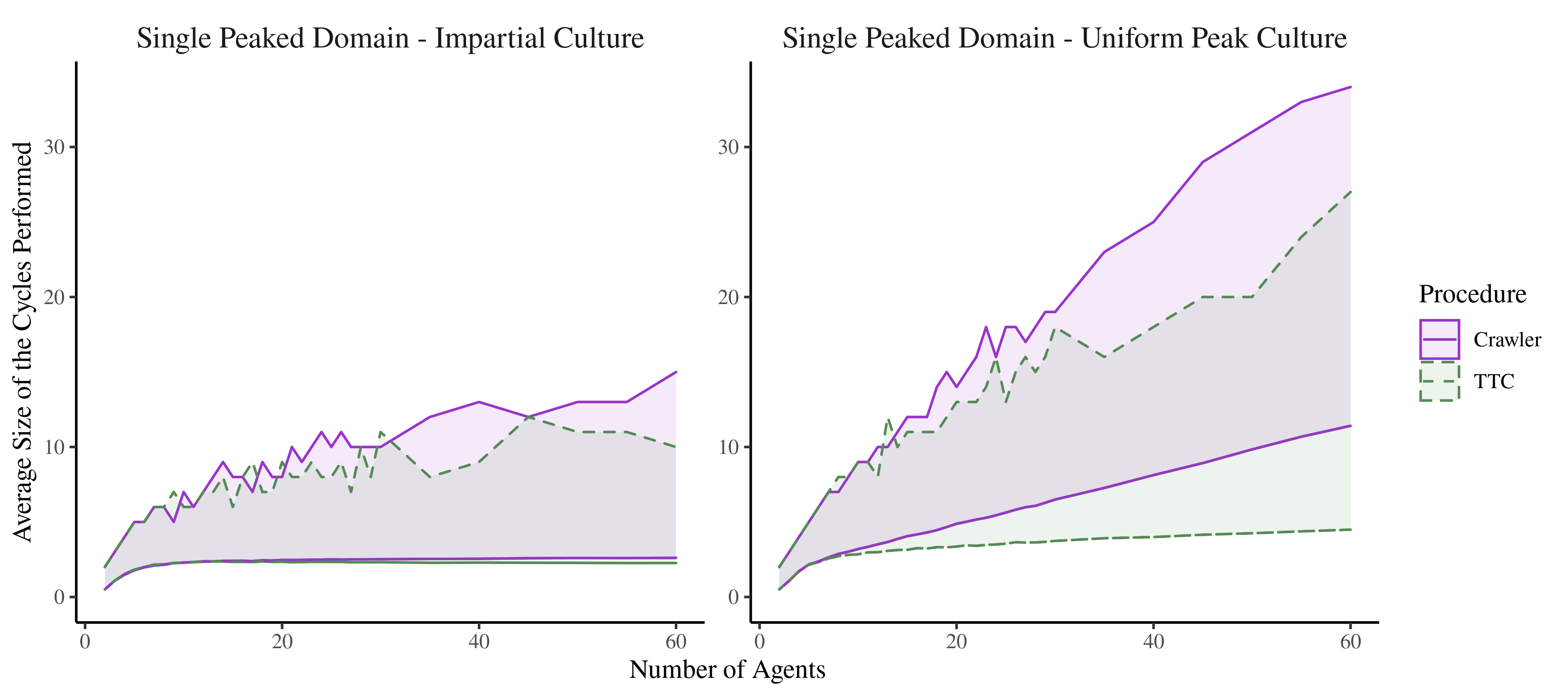}
	\caption{Average size of the cycles performed by TTC and Crawler, the filled area represents the maximum. The minimum, not shown, is constantly equal to zero. The number of agents varies from 2 to 60.}
	\label{fig:expeCycleSize}
\end{figure*}

\paragraph{Comparison with TTC and  Crawler.}
\revisedC{
The number of deals performed by swap dynamics can be compared with the number of deals induced by TTC or  Crawler. To do so we introduce the number of unit exchanges performed by a procedure. It corresponds to the sum of the length of the cycle-deals applied. For a swap dynamics, it is twice the number of swap-deals performed for instance. This measure allows us to compare all the procedures even though they perform cycle-deals of different sizes. As shown in Figure \ref{fig:expeNbExchanges}, the procedure can be clustered in three distinct classes. The swap dynamics can be grouped together, offering a middle way between TTC which performs particularly well and Crawler which performs particularly badly in terms of number of unit exchanges.
}

\revisedC{
The difference between Crawler and TTC can be explained by the fact that the former performs much more deals on average than TTC. As shown in Figure \ref{fig:expeNbCycle}, this can be observed for any number of agents and the gap increases as the number of agents increases.  The size of the deals is also larger when implementing Crawler  (Figure \ref{fig:expeCycleSize}). Both of these facts naturally lead to a higher number of unit exchanges.
}
	
	The larger number of deals and the larger  sizes of deals performed by Crawler are related to the fact that agents are ordered with respect to the resource they initially hold and with respect to the order of these resources on the single-peaked axis. Based on this order,  Crawler only considers deals $\mu = \tuple{a_1, \ldots, a_k}$ such that $a_i$ and $a_{i+1}$ (with $i \in \{1, 2, \cdots k-1\}$) are owners of adjacent resources on the axis.  Hence, TTC allows for considering  a larger range of cycle-deals than  Crawler. 

Putting aside the comparison between  Crawler, and TTC, the size of the deals to implement can be extremely large for both procedures as depicted in Figure \ref{fig:expeCycleSize}. 
A cycle-deal may involve more than half (respectively 35\%) of the agents under UP-SP and a fifth (respectively 15\%) of the agents under IC-SP for  Crawler (respectively TTC). The linear regressions explaining these values are presented in Table \ref{tab:regressionMaxSize}. \revisedC{Overall, even though TTC performs less unit exchanges than our swap dynamics, it can still need to perform extremely large cycle-deals.}

\begin{table}
	\centering
	\begin{tabular}{cccccc}
		\toprule
		& TTC & TTC && Crawler & Crawler \\
		& IC-SP & UP-SP && IC-SP & UP-SP \\
		\midrule
		$\beta_0$ & 4.64 & 4.73 && 4.27 & 2.86 \\
		$\beta_1$ & 0.15 & 0.36 && 0.20 & 0.56 \\
		$R^2$ & 0.7872 & 0.9032 && 0.883 & 0.9862 \\
		\bottomrule
	\end{tabular}
	\caption{Linear regression of the maximum size of the deals over the number of agents. $\beta_0$ and $\beta_1$ are the coefficients of the regression: $maxSize = \beta_0 + \beta_1 * n$ where $n$ is the number of agents, $R^2$ the coefficient of determination. The p-value is omitted as it is not meaningful for simulations (see, e.g., \cite{lee2015complexities}).}
	\label{tab:regressionMaxSize}
\end{table}


	These experiments show that the number of swap-deals can be significant, especially under UP-SP where long sequences of slightly improving steps are more likely to occur. On the other hand, we show empirically that TTC and Crawler are indeed prone to require the implementation of cycle-deals of large size. This illustrates the trade-off which occurs between the coordination requirements and the length of the procedure.

\section{Conclusion}
\label{sec:conclu}

This paper studied the property of swap dynamics for the allocation of indivisible resources in the restricted setting of single-peaked housing markets. The basic principle of the procedure is to let agents perform pairwise improving exchanges. We showed in particular that the single-peaked domain happens to be maximal for guaranteeing convergence to Pareto-optimal outcomes with such dynamics. We also showed that the outcomes of TTC and of Crawler are reachable by swap-deal sequences. 
To refine our analysis, we have studied two further notions: the average rank and the minimum rank of the resources obtained by the agents. None of the procedures discussed in this paper are specifically designed for optimizing  these ranks, even though these notions capture very natural criteria of efficiency (for the average rank) and fairness (for the minimum rank). It thus seems important to study how these allocation procedures behave on that respect. 
To complement worst-case theoretical bounds on the loss of social welfare induced by swap dynamics, we ran experiments which demonstrated that they actually provide good results in practice. 


	Our focus on swap-deals is motivated by their minimal coordination requirements. In the context of this paper, our experimental results suggest that there is little gain to expect when allowing deals involving three agents. A complementary relaxation is to permit deals possibly not improving for some of the agents involved (but at least a  majority). 
	This may unveil interesting connections with other notions:  \cite{kondratev2019minimal} recently established for instance that an allocation is popular in a housing market if and only if no (majority) improving exchanges between three agents exist. 

To go further with the experiments, it would be interesting to use real data. Our attempt to use data from Preflib \citep{mattei2013preflib} was not successful as there is no dataset that is single-peaked when there are more than 5 agents. Getting such preferences would be an interesting way to confirm our results. 
Regarding the model itself, \cite{bade2019matching} extended  Crawler  to single-peaked domains with indifferences. Whether our results with swap dynamics could be similarly generalized is an avenue for future research. 
Overall, this paper raises the exciting issue of giving a characterization of rules that are efficient, individually rational and strategy-proof for the single-peaked domain. Such a characterization would also provide more insights about the type of selection heuristic for which the swap dynamics are strategy-proof. It would also be interesting to tackle the characterization of the swap dynamics efficiency, that is, identifying the domain on which the swap dynamics are Pareto-optimal. Our maximality result is a significant step in this sense, it would be nice to complete the picture.


\section*{Acknowledgements}
We thank Sophie Bade, Yann Chevaleyre, Anastasia Damamme, and Julien Lesca, for discussions related to this topic as well as the anonymous reviewers for their comments and suggestions which significantly improved the paper.


\bibliographystyle{spbasic}


\end{document}